\documentclass{article} % For LaTeX2e
\usepackage{iclr2023_conference,times}

\usepackage{centernot}

\usepackage[utf8]{inputenc} % allow utf-8 input
\usepackage[T1]{fontenc}    % use 8-bit T1 fonts
\usepackage[colorlinks=true]{hyperref}       % hyperlinks
\usepackage{url}            % simple URL typesetting
\usepackage{booktabs}       % professional-quality tables
\usepackage{amsfonts}       % blackboard math symbols
\usepackage{nicefrac}       % compact symbols for 1/2, etc.
\usepackage{microtype}      % microtypography
\usepackage{xcolor}         % colors
\usepackage{subcaption} %% For complex figures with subfigures/subcaptions
\usepackage{listings}
\usepackage{relsize}
\usepackage{enumerate}
\usepackage{float}
\usepackage{graphicx}
\usepackage{multirow}
\usepackage[algosection,ruled,vlined,linesnumbered,algo2e]{algorithm2e}
  \SetKwInput{KwIn}{Inputs}
  \SetKwFunction{Ffilter}{Filter}
  \SetKwFunction{Fattack}{Attack}
  \SetKwFunction{Feval}{Eval}
  \SetKwFunction{Ffor}{for}
  \SetKwProg{Fn}{}{:}{}
  \DontPrintSemicolon
\usepackage{wrapfig}
\usepackage{amsmath}
\usepackage{amssymb}
\usepackage{mathtools}
\usepackage{amsthm}
\usepackage{dsfont}
\usepackage{wasysym}
\usepackage{fixfoot}
\usepackage[subtle]{savetrees}

\theoremstyle{plain}
\newtheorem{theorem}{Theorem}[section]

\newtheorem{example}[theorem]{Example}
\theoremstyle{definition}
\newtheorem{definition}[theorem]{Definition}

\theoremstyle{remark}

\definecolor{blue2}{rgb}{0.0, 0.5, 1.0}

\newcommand{\ClassSet}{\mathcal{Y}}
\newcommand{\notrobust}{0}
\newcommand{\robust}{1}
\newcommand{\certify}[1]{\Tilde{#1}}
\newcommand{\id}[1]{{(#1)}}
\newcommand{\cas}{C}
\newcommand{\pcas}{P}
\newcommand{\vote}{V}
\newcommand{\uvote}{U}

\newcommand{\cond}{\texttt{c}}

\newcommand{\ensembler}{\mathcal{E}}
\newcommand{\idxs}{\texttt{idxs}}
\newcommand{\celoss}{\mathcal{L}^{ce}}
\newcommand{\unifdist}{\textsf{unif}}
\newcommand{\onehot}{\textsf{one-hot}}

\newcommand{\pcondcert}{\texttt{c2}}
\newcommand{\pcondclean}{\texttt{c1}}

\newcommand{\indicator}{\mathds{1}}
\newcommand{\lab}[1]{\ensuremath{#1_\textsf{label}}}
\newcommand{\logit}[1]{\ensuremath{#1_\textsf{logit}}}
\newcommand{\cer}[1]{\ensuremath{#1_\textsf{cert}}}

\newcommand{\cra}{\texttt{CRA}}
\newcommand{\acc}{\texttt{Acc}}
\newcommand{\fpr}{\texttt{FPR}}

\newcommand{\era}{\texttt{ERA}}

\DeclareMathOperator*{\argmax}{argmax}
\DeclareFixedFootnote{\algnote}{} 
\DeclareFixedFootnote{\tabnote}{}

\title{On the Perils of Cascading Robust Classifiers}

\author{%
Ravi Mangal$^*$, Zifan Wang$^*$, Chi Zhang\thanks{Equal Contribution}\\
  Electrical and Computer Engineering\\
  Carnegie Mellon University\\
  Pittsburgh, PA 15213 \\
\texttt{\{rmangal, zifanw, chiz5\}@andrew.cmu.edu}
\AND
  Klas Leino\\
  School of Computer Science\\
  Carnegie Mellon University\\
  Pittsburgh, PA 15213 \\
\texttt{kleino@cs.cmu.edu}
\And
  Corina P\u{a}s\u{a}reanu \\
  Carnegie Mellon University \\
  and NASA Ames \\
  Moffett Field, CA 94043 \\
\texttt{pcorina@andrew.cmu.edu}
  \And
  Matt Fredrikson \\
  School of Computer Science\\
  Carnegie Mellon University\\
  Pittsburgh, PA 15213 \\
\texttt{mfredrik@cmu.edu}
}

\iclrfinalcopy % Uncomment for camera-ready version, but NOT for submission.

\begin{document}

\maketitle

\begin{abstract}
% !TEX root=./main.tex

Ensembling certifiably robust neural networks is a promising approach for improving the \emph{certified robust accuracy} of neural models. 
Black-box ensembles that assume only query-access to the constituent models (and their robustness certifiers) during prediction are particularly attractive due to their modular structure. 
Cascading ensembles are a popular instance of black-box ensembles that appear to improve certified robust accuracies in practice. 
However, we show that the robustness certifier used by a cascading ensemble is unsound. 
That is, when a cascading ensemble is certified as locally robust at an input $x$ (with respect to $\epsilon$), there can be inputs $x'$ in the $\epsilon$-ball centered at $x$, such that the cascade's prediction at $x'$ is different from $x$ and thus the ensemble is not locally robust. 
Our theoretical findings are accompanied by empirical results that further demonstrate this unsoundness. We present \emph{cascade attack} (CasA), an adversarial attack against cascading ensembles, and show that: (1) there exists an adversarial input for up to 88\% of the samples where the 
ensemble claims to be certifiably robust and accurate; and (2) the accuracy of a cascading ensemble under our attack is as low as 11\% when it claims to be certifiably robust and accurate on 97\% of the test set. 
Our work reveals a critical pitfall of cascading certifiably robust models by showing that the seemingly beneficial strategy of cascading can actually hurt the robustness of the resulting ensemble. Our code is available at \url{https://github.com/TristaChi/ensembleKW}.

% We present an alternate black-box ensembling mechanism based on weighted voting which we prove to be sound for robustness certification. 
% Via a thought experiment, we demonstrate that if the constituent classifiers are suitably diverse, voting ensembles can improve certified performance. 
%Our code is available at \url{https://github.com/TristaChi/ensembleKW}.
%But such diversity is to hard to find in practice.

\end{abstract}

% !TEX root=./main.tex

\section{Introduction}
\label{sec:intro}
%- Limitations of current approaches to certified adversarial robustness \\
%- Ensembling as a generic idea for combining weak learners \\
%- Ensembles for robustness - certified or empirical \\
%- We show that existing certifiably robust ensembles are broken \\
%- We propose a correct ensembling scheme for certifiable robustness \\
%- For our certified ensemble to do well, the ensemble constituents need to be diverse yet overlapping \\
%- For diversity, typical approaches are to use sample weights (what other approaches?) \\
%- Empirical observation - certifiably robust training methods are not sensitive to sample weights compared to std training or adversarial training (Madry et al.)
%- Demonstrated for GloRo, Cayley, BCP (?), Kolter-Wong (?)

Local robustness has emerged as an important requirement of classifier models. It ensures that models are not susceptible to misclassifications caused by small perturbations to correctly classified inputs.
A lack of robustness can be exploited by not only malicious actors (in the form of adversarial examples \citep{szegedy13}) but can also lead to incorrect behavior in the presence of natural noise \citep{gilmer19a}.
However, ensuring local robustness of neural network classifiers has turned out to be a hard challenge. Although neural networks can achieve state-of-the-art classification accuracies on a variety of important tasks,
neural classifiers with comparable certified robust accuracies\footnote{Percentage of inputs where the classifier is accurate and certified as locally robust.} ($\cra$, Def.~\ref{def:cra}) remain elusive, even when trained in a robustness-aware manner \citep{madry18,wong2018provable,cohen19,leino21}. 
In light of the limitations of robustness-aware training, ensembling certifiably robust neural classifiers has been shown to be a promising approach for improving certified robust accuracies \citep{wong2018scaling,yang2022on}. 
An ensemble combines the outputs of multiple base classifiers to make a prediction, and is a well-known mechanism for improving classification accuracy when one only has access to weak learners \citep{dietterich2000,bauer1999empirical}.

Ensembles designed to improve $\cra$ take one of two forms. 
\emph{White-box ensembles} \citep{yang2022on,zhang2019enhancing,liu2020enhancing} assume white-box access to the constituent models. They calculate new logits by averaging the corresponding logits of the constituent classifiers. 
For local robustness certification, they treat the ensemble as a single, large model and then use off-the-shelf techniques  \citep{cohen19,weng18,wong2018provable,zhang18efficient} for certification.
\emph{Black-box ensembles} \citep{wong2018scaling,blum2022boosting}, on the other hand, assume only query-access to the constituent classifiers during prediction, and are, therefore, agnostic to their internal details.
They re-use the prediction and certification outcomes of the constituent models to calculate the ensemble's prediction and certificate. Their black-box nature lends them modularity and permits any combination of constituent classifiers, irrespective of their individual certification mechanism, so we focus our efforts on them in this paper. 

\emph{Cascading ensembles} \citep{wong2018scaling,blum2022boosting} are a particularly popular instance of black-box ensembles that appear to improve $\cra$ in practice. They evaluate the constituent classifiers (and their certifiers) in a fixed sequence.
The ensemble's prediction is the output of the first constituent classifier in the sequence that is certified locally robust, defaulting to the last classifier's output if no model can be certified.
Importantly, the cascading ensemble is itself certified locally robust only when at least one of the constituent classifiers is certified locally robust.

\begin{figure}[t]
    \centering
    \includegraphics[width=\textwidth]{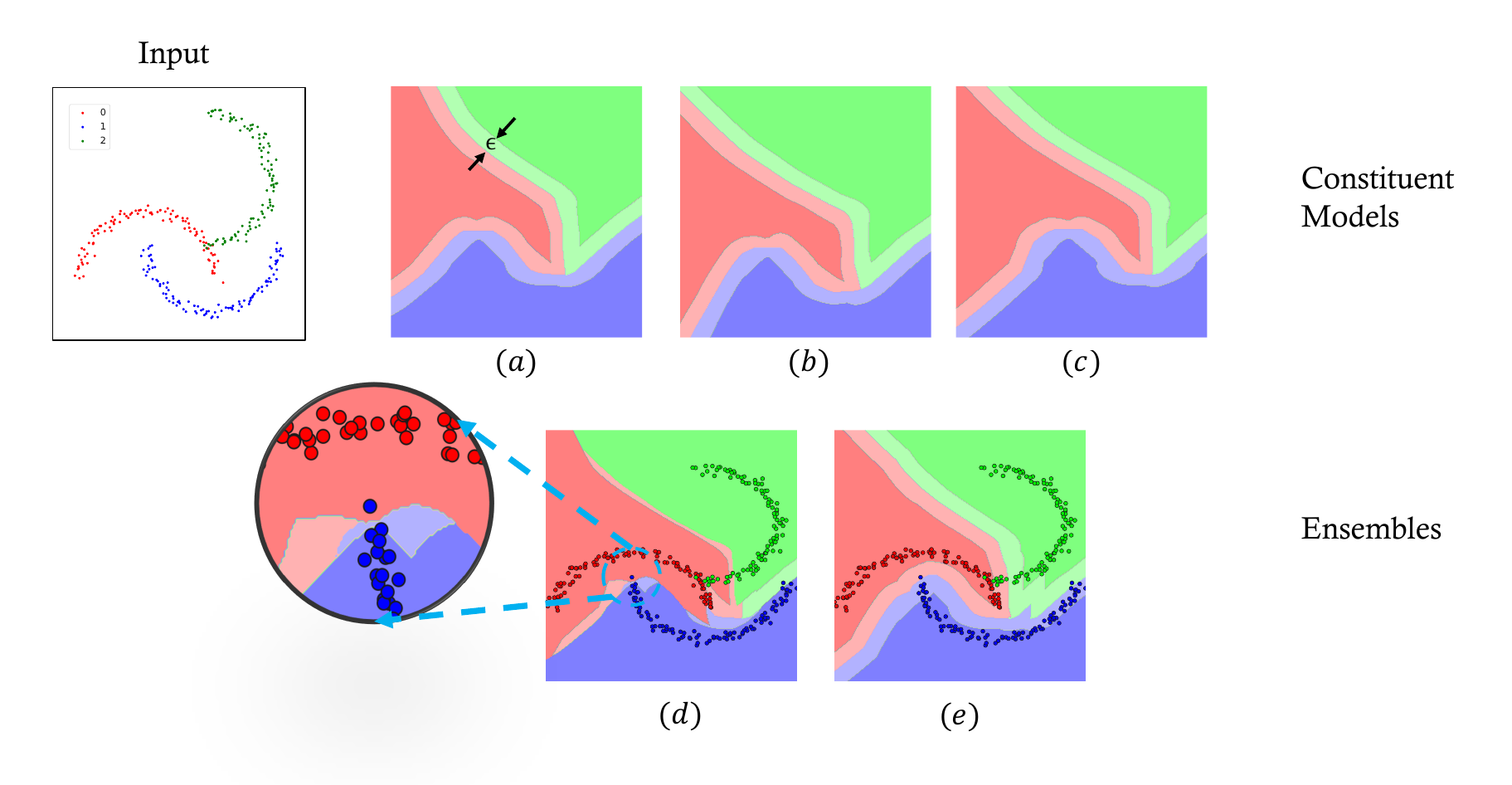}
    \caption{Visualizing classification results of 2D points for constituent models (a-c) and the corresponding Cascading Ensemble (d, Def.~\ref{def:cascading_ensembler}) and Uniform Voting Ensemble (e, Def.~\ref{def:unif_voting_ensembler}). Regions with colors correspond to predictions (0: red, 1: blue, 2: green) made by the underlying model (or ensemble). Darker colors indicate that the accompanying robustness certification of the underlying model (or ensemble) returns 1 and lighter colors are for cases when the certification returns 0. All points receiving 1 for certifications (darker regions) are at least $\epsilon$-away from the other classes in (a)-(c), i.e. \emph{certification is sound} (Def.~\ref{def:soundness-of-F_Tilde}). This property is violated in (d), e.g. points from dark red regions are not $\epsilon$-away from the blue region in the zoomed-in view on the left, but preserved in (e). Namely, voting ensembles are \emph{soundness-preserving} (Def.~\ref{def:sp_ensembler}) while cascading ensembles are not. }  
    \label{fig:2d-example}
\end{figure}

\textbf{Our contributions.} We show in this paper that the local robustness certification mechanism used by cascading ensembles is unsound even when the certifiers used by each of the constituent classifiers are sound (Theorem~\ref{thm:cascading_sound}). In other words, when a cascading ensemble is certified as locally robust at an input $x$, there can, in fact, be 
inputs $x'$ in the $\epsilon$-ball centered at $x$, such that the cascade's prediction at $x'$ is different from $x$. 
Figure~\ref{fig:2d-example} demonstrates this visually on a toy dataset. The cascading ensemble can have points that are less than $\epsilon$ away from the decision boundary, yet the ensemble is certified locally robust at such points (Figure~\ref{fig:2d-example}(d)).
As a consequence of our result, use of a cascading ensemble in any scenario requiring local robustness guarantees is unsafe and existing empirical results that report the $\cra$ of cascading ensembles are not valid.
%Though prior works using cascading ensembles never explicitly claim that the cascade robustness certification mechanism is sound, their comparisons of cascade performance against that of a single \emph{certifiable} model using $\cra$ as a metric can lead an uncareful reader to believe that the cascade really is certifiable.

Guided by our theoretical construction, we propose \emph{cascade attack} (CasA, Algorithm~\ref{alg:attack}), an adversarial attack against cascading ensembles, and conduct an empirical evaluation with the cascading ensembles trained by \citet{wong2018scaling} for MNIST and CIFAR-10 datasets.
With CasA, we show that: (1) there exists an adversarial example for up to 88\% of the samples where the ensemble claims to be certifiably robust and accurate; (2) the empirical robust accuracy of a cascading ensemble is as low as 11\% while it claims to be certifiably robust and accurate on 97\% of the test set; and (3) viewing all experiments as a whole, the empirical robust accuracy of a cascading ensemble is almost always lower than even the $\cra$ of the single best model in the ensemble. Namely, a cascading ensemble is often less robust.
Our results conclusively demonstrate that the unsoundness of the cascading ensemble certification mechanism can be exploited in practice, and cause the ensemble to perform markedly worse than the single best constituent model.

We also present an alternate ensembling mechanism based on weighted voting that, like cascading ensembles, assumes only query-access to the constituent classifiers but comes with a provably sound local robustness certification procedure (Section~\ref{sec:defense}). 
%Our proposed ensemble is inspired by voting-based ensembles that conduct a vote amongst the constituent classifiers but our soundness proof is novel.
We show through a thought experiment that it is possible for a voting ensemble to improve upon the $\cra$ of its constituent models (Section~\ref{sec:thought}), and observe that the key ingredient for the success of voting ensembles is a suitable balance between diversity and similarity of their constituents. 
We leave the design of training algorithms that balance the diversity and similarity of the constituent classifiers as future work.

%To measure the extent to which the unsoundness of the cascading ensemble distorts the ensemble $\cra$, we repeat the experiments by \citet{wong2018scaling} using our sound, voting ensemble (instead of their cascading ensemble) (Section~\ref{sec:hardness_eval}). 
%With the voting ensemble, all the $\cra$ gains of the cascading ensemble over the constituent classifiers are wiped out. 
%These results, however, should not be interpreted to mean that voting ensembles cannot improve upon the $\cra$ of the constituent classifiers. 
%While we were unable to obtain an improvement on a realistic data set, we show through a thought experiment that it is possible for a voting ensemble to outperform their constituent models (Section~\ref{sec:thought}).
%The key ingredient for the success of voting ensembles is a suitable balance between diversity and similarity of their constituents. 
%This observation opens a promising direction for future research, namely, the design of training algorithms that balance the diversity and similarity of the constituent classifiers informed by the needs of our voting ensemble.
%\todo{say something about sparse weights?}%\todo{say something about sparse weights?}
%\todo{should we say something like: While we were unable to obtain an improvement using a realistic data set, we show through a thought experiment that it is possible to improve the \cra through a voting ensemble?}

% !TEX root=./main.tex

\section{Cascading Ensembles}
\label{sec:unsound}

In this section, we introduce our notation and required definitions. We then show the local robustness certification procedure used by cascading ensembles is unsound.
%we demonstrate how local robustness guarantee is violated when locally-robust models are sequentially ensembled. 

\subsection{Certifiable Classifiers and Ensemblers}\label{sec:cascade-local-robustness}
We begin with our notation and necessary definitions. Suppose a neural network $f: \mathbb{R}^d \rightarrow \mathbb{R}^m$ takes an input and outputs the probability of $m$ different classes. The subscript $x_j$ denotes the $j$-th element of a vector $x$. When discussing multiple networks, we differentiate them with a superscript, e.g. $f^\id{1}, f^\id{2}$. Throughout the paper we use the upper-case letter $F$ to denote the prediction of $f$ such that $F(x) = \arg\max_{j \in \ClassSet}\{f_j(x)\}$ where $f_j$ is the logit for class $j$ and $\ClassSet=[m]$.\footnote{$[m]:=\{0, 1, ..., m-1\}$} 
The prediction $F(x)$ is considered $\epsilon$-locally robust at $x$ if all neighbors within an $\epsilon$-ball centered at $x$ receive the same predictions, which is formally stated in Def.~\ref{def:locally-robust}.

\begin{definition}[$\epsilon$-Local Robustness]\label{def:locally-robust}
A network, $F$, is $\epsilon$-locally robust at $x$ w.r.t to norm, $||\cdot||$, if $\forall x' ~.~ ||x'-x|| \leq \epsilon \Longrightarrow F(x') = F(x)$.
% \begin{align*}
%     \forall x' ~.~ ||x'-x|| \leq \epsilon \Longrightarrow F(x') = F(x)
% \end{align*}
\end{definition}

%Certifying the local robustness of a network has been receiving increasing attention from the community as the naive way to verify Def.~\ref{def:locally-robust}, i.e. checking all possible neighbors in the $\epsilon$-ball, is not a practical choice. These smarter certification methods instead rely on solving a corresponding linear~\cite{GeoCERT} or a semi-definite programs~\cite{}, interval propagations~\cite{IBP, BCP, CROWN}, gradient projections~\cite{FGP}, dual networks~\cite{KW}, or Lipschitz computations~\citep{leino21}. 
Though local robustness certification of ReLU networks is NP-Complete \citep{katz17}, due to its importance, the problem has been receiving increasing attention from the community. Proposed certification methods rely on a variety of algorithmic approaches like  
solving corresponding linear~\citep{jordan19provable} or semi-definite programs~\citep{raghunathan2018certified}, interval propagations~\citep{gowal2018effectiveness, lee20lipschitz, zhang18efficient}, abstract interpretation~\citep{singh19}, geometric projections~\citep{fromherz20}, dual networks~\citep{wong2018provable}, or Lipschitz approximations~\citep{leino21,weng18}.

If a certification method is provided for a network $F$, we use $\certify{F}^\epsilon: \mathbb{R}^d \rightarrow \ClassSet \times \{\notrobust, \robust\}$ to denote a \emph{certifiable neural classifier} that returns a prediction according to $F$ and the outcome of the certification method applied to $F$ with respect to robustness radius $\epsilon$. 
We use $\lab{\certify{F}^\epsilon}(x)$ to refer to the prediction and $\cer{\certify{F}^\epsilon}(x)$ to refer to the certification outcome.  
If $\cer{\certify{F}^\epsilon}(x) = \notrobust$, the accompanying robustness certification is unable to certify $F$ (i.e., $\lab{\certify{F}}$) as $\epsilon$-locally robust at $x$. 
When $\epsilon$ is clear from the context, we directly write $\certify{F}$. One popular metric to evaluate the performance of any $\certify{F}$ is Certified Robust Accuracy (\cra).

\begin{definition}[Certified Robust Accuracy]
	\label{def:cra}
		The certified robust accuracy ($\cra$) of a certifiable classifier $\certify{F} \in \mathbb{R}^d \rightarrow \ClassSet \times \{\notrobust, \robust\}$ on a given dataset $S_k \subseteq \mathbb{R}^d\times \ClassSet$ with $k$ samples is given by $\cra(\certify{F},S_k) := 1/k \sum_{(x_i,y_i) \in S_k}\indicator[\certify{F}(x_i)=(y_i,\robust)] $.
		% $$
		% \cra(\certify{F},S_k) := \frac{1}{k} \sum_{(x_i,y_i) \in S_k}\indicator\left[\certify{F}(x_i)=(y_i,\robust) \right] 
		% $$
	\end{definition}
For $\certify{F}$ and its \cra~to be useful in practice without providing false robustness guarantees, it must be \emph{sound} to begin with (Def.~\ref{def:soundness-of-F_Tilde}). 

% A certification method $\certify^\epsilon$ returns either $\robust$ or $\notrobust$ that indicate whether the prediciton $F(x)$ is $\epsilon$-LR or not and is said to be complete if this works out for every point (Def.~\ref{def:soundness-of-certification}).

\begin{definition}[Certification Soundness]~\label{def:soundness-of-F_Tilde} A certifiable classifier, $\certify{F}: \mathbb{R}^d \rightarrow \ClassSet \times \{\notrobust, \robust\}$,  is sound if $\forall x \in \mathbb{R}^d ~.~ \cer{\certify{F}}(x) = \robust  ~~\Longrightarrow~~ \lab{\certify{F}} \text{ is } \epsilon\text{-locally robust at } x$.
% \begin{align*}
% 	\forall x \in \mathbb{R}^d ~.~ \cer{\certify{F}}(x) = \robust  ~~\Longrightarrow~~ \lab{\certify{F}} \text{ is } \epsilon\text{-locally robust at } x.
% \end{align*}
\end{definition}
Notice that if there exist $\epsilon$-close inputs $x, x'$ where $\certify{F}(x) = (y, \robust)$ and $\certify{F}(x') = (y', \notrobust)$, where $y \neq y'$, then it still means that $\certify{F}$ is not sound. We define an \emph{ensembler} (Def.~\ref{def:ensembler}) as a function that combines multiple certifiable classifiers into a single certifiable classifier (i.e., an \emph{ensemble}).
\begin{definition}[Ensembler]\label{def:ensembler}
	Let $\mathbb{\certify{F}} := \mathbb{R}^d \rightarrow \ClassSet \times \{\notrobust, \robust\}$ represent the set of all certifiable classifiers.
	An ensembler  $\ensembler: \mathbb{\certify{F}}^N \rightarrow \mathbb{\certify{F}}$ is a function over $N$ certifiable classifiers that returns a certifiable classifier.
\end{definition}

A \emph{query-access} ensembler formalizes our notion of a black-box ensemble.
\begin{definition}[Query-Access Ensembler]\label{def:qa_ensembler}
	Let $\mathbb{G} := (\ClassSet \times \{\notrobust, \robust\})^N \rightarrow \ClassSet \times \{\notrobust, \robust\}$.
	$\ensembler$ is a query-access ensembler if,
	$ \forall \certify{F}^\id{0}, \certify{F}^\id{1}, \ldots, \certify{F}^\id{N-1} \in \mathbb{\certify{F}}~.~\exists G \in \mathbb{G}~.~\forall x \in \mathbb{R}^d.$ 
	$$
		 \ensembler\left(\certify{F}^\id{0}, \certify{F}^\id{1}, \ldots, \certify{F}^\id{N-1}\right)(x) ~=~ G\left(\certify{F}^\id{0}(x), \certify{F}^\id{1}(x), \ldots, \certify{F}^\id{N-1}(x)\right)
	$$
\end{definition}

Def.~\ref{def:qa_ensembler} says that if $\ensembler$ is \emph{query-access} its output $\certify{F}$ can always be re-written as a function over the outputs of the certifiable classifiers $\certify{F}^\id{0}, \certify{F}^\id{1}, \ldots, \certify{F}^\id{N-1}$. 
Put differently, $\certify{F}$ only has black-box or query-access to classifiers  $\certify{F}^\id{0}, \certify{F}^\id{1}, \ldots, \certify{F}^\id{N-1}$.

Finally, a \emph{soundness-preserving} ensembler (Def.~\ref{def:sp_ensembler}) ensures that if the constituent certifiable classifiers are sound (as defined in Def.~\ref{def:soundness-of-F_Tilde}), the ensemble output by the ensembler is also sound.
\begin{definition}[Soundness-Preserving Ensembler]\label{def:sp_ensembler}
	An ensembler $\ensembler$ is soundness-preserving if,
	$\forall \certify{F}^\id{0}, \certify{F}^\id{1}, \ldots, \certify{F}^\id{N-1} \in \mathbb{\certify{F}}, \certify{F}^\id{0}, \certify{F}^\id{1}, \ldots, \certify{F}^\id{N-1}~\text{are sound} ~~\Longrightarrow~~ \ensembler(\certify{F}^\id{0}, \certify{F}^\id{1}, \ldots, \certify{F}^\id{N-1})~\text{is sound}$.
	% $$
	% \certify{F}^\id{0}, \certify{F}^\id{1}, \ldots, \certify{F}^\id{N-1}~\text{are sound} ~~\Longrightarrow~~ \ensembler(\certify{F}^\id{0}, \certify{F}^\id{1}, \ldots, \certify{F}^\id{N-1})~\text{is sound}
	% $$
\end{definition}

\subsection{Cascading Ensembler Is Not Soundness-Preserving}
\label{sec:cascade-unsoundness}
Cascading ensembles \citep{wong2018scaling, blum2022boosting} are a popular instance of black-box ensembles that appear to be practically effective in improving certified robust accuracies. However, we show that cascading ensembles are not sound.

We define a cascading ensemble to be the output of a \emph{cascading ensembler} (Def.~\ref{def:cascading_ensembler}).
A cascading ensemble evaluates its constituent certifiable classifiers in a fixed sequence. For a given input $x$, the ensemble either returns the prediction and certification outcomes of the first constituent classifier $\certify{F}^\id{j}$ such that 
$\cer{\certify{F}}^\id{j}=\robust$ or of the last constituent classifier in case none of the constituent classifiers can be certified. Clearly, cascading ensemblers are query-access (formal proof in Appendix~\ref{sec:proofs}).

\begin{definition}[Cascading Ensembler]\label{def:cascading_ensembler}
    Let $\certify{F}^\id{0}, \certify{F}^\id{1}, \ldots, \certify{F}^\id{N-1}$ be $N$ certifiable classifiers. A cascading ensembler $\ensembler_\cas: \mathbb{\certify{F}}^N \rightarrow \mathbb{\certify{F}}$ is defined as follows
    \begin{align*}
	    \ensembler_\cas\left(\certify{F}^\id{0}, \certify{F}^\id{1}, \ldots, \certify{F}^\id{N-1}\right)(x) ~:=~
        \begin{cases} 
		\certify{F}^\id{j}(x) & \text{if } \exists j \leq N-1~.~\cond(j)=1 \\
		\certify{F}^\id{{N-1}}(x) & \text{otherwise}
        \end{cases}
        \end{align*}
	where $\cond(j) := 1 \text{ if } (\cer{\certify{F}^\id{j}}(x) = \robust) \text{ and } (\forall i < j, \cer{\certify{F}^\id{i}}(x) = \notrobust)$, and 0 otherwise.
\end{definition}

%It is straightforward to show that cascading ensemblers are query-access (Thm.~\ref{thm:cascading_qa}).
%\begin{theorem}
%\label{thm:cascading_qa}
%	Cascading ensembler $\ensembler_\cas$ is query-access. 
%\end{theorem}
%\begin{proof}
%	See Appendix~\ref{sec:proofs} in supplemental material.
%\end{proof}

Theorem~\ref{thm:cascading_sound} shows that cascading ensemblers are not soundness-preserving, and so a cascading ensemble can be unsound. We show this by means of a counterexample.
\begin{theorem}  
\label{thm:cascading_sound}
	The cascading ensembler $\ensembler_\cas$ is not soundness-preserving. 
\end{theorem}
\begin{proof}
We can re-write the theorem statement as,
$\exists \certify{F}^\id{0}, \certify{F}^\id{1}, \ldots, \certify{F}^\id{N-1} \in \mathbb{\certify{F}}$ such that for $\certify{F} := \ensembler_\cas(\certify{F}^\id{0}, \certify{F}^\id{1}, \ldots, \certify{F}^\id{N-1}),
	    \exists x \in \mathbb{R}^d, \cer{\certify{F}}(x) = \robust  ~\centernot{\Longrightarrow}~ \lab{\certify{F}} \text{ is } \epsilon\text{-locally robust at } x$.

We prove by constructing the following counterexample. Consider a cascading ensemble $\certify{F}$ constituted of certifiable classifiers $\certify{F}^\id{0}$ and $\certify{F}^\id{1}$.
$\certify{F}^\id{0}$ and $\certify{F}^\id{1}$ are such there exists an $x$ where

\begin{equation}
\label{eq:cas_sound1}
	\left(\cer{\certify{F}}^\id{0}(x) = \notrobust\right)~\wedge~\left(\certify{F}^\id{1}(x) = (y,\robust)\right)
\end{equation}

Using Def.~\ref{def:cascading_ensembler}, it is true that $\certify{F}(x) = (y,\robust)$.
%\begin{align*}
%	\certify{F}(x) = (y,\robust)
%\end{align*}

Without violating (\ref{eq:cas_sound1}), we can have another point $x'$ such that, 

\begin{equation}
\label{eq:cas_sound2}
	(||x-x'|| \leq \epsilon)~\wedge~(\certify{F}^\id{0}(x') = (y',\robust))~\wedge~(\certify{F}^\id{1}(x') = (y,\notrobust))~\wedge~(y' \neq y)
\end{equation}

Using Def.~\ref{def:cascading_ensembler}, it is true that $\certify{F}(x') = (y',\robust)$.
%\begin{align*}
%	\certify{F}(x') = (y',\robust)
%\end{align*}

Thus, for two points $x, x'$ constructed as above, we show that $\exists x, x', \text{ s.t. } ||x'-x|| \leq \epsilon$, $\cer{\certify{F}}(x) = \robust  ~~\centernot{\Longrightarrow}~~  \lab{\certify{F}}(x') = \lab{\certify{F}}(x)$,
% \begin{align*}
% 	\cer{\certify{F}}(x) = \robust  ~~\centernot{\Longrightarrow}~~  \lab{\certify{F}}(x') = \lab{\certify{F}}(x)
% \end{align*}
which violates the condition of local robustness (Def.~\ref{def:locally-robust}).
\end{proof}

The counterexample constructed in Thm.~\ref{thm:cascading_sound} is not just hypothetical, but something that materializes on real models (see Figure~\ref{fig:2d-example} for a toy example and Section~\ref{sec:eval} for our empirical evaluation).

% !TEX root=./main.tex
\section{Attacking Cascading Ensembles}
\label{sec:attack}
\begin{algorithm2e}[t]
\small
\vspace{0.5em}
	\KwIn{Ensemble $\certify{F} \in \mathbb{\certify{F}}$, constituent models $\certify{F}^\id{0}$, $\ldots$,  $\certify{F}^\id{N-1} \in \mathbb{\certify{F}}$, input $x \in \mathbb{R}^d$, 
	%index $j \in [N]$ of the constituent used for prediction, 
	attack bound $\epsilon \in \mathbb{R}$, and distance metric $\ell_p$}
\KwOut{An attack input $x' \in \mathbb{R}^d$}
\vspace{0.5em}
\Fn{\Fattack{$\certify{F}, \certify{F}^\id{0}~,~\ldots~,~\certify{F}^\id{N-1}~,~x~,~\epsilon~,~\ell_p$}}{
	$y ~:=~ \lab{\certify{F}}(x)$\;
	$\idxs ~:=~ \{i~|~i \in [N] \wedge \lab{\certify{F}^\id{i}}(x) \neq y\}~\cup~\{N-1\}$\;
	\ForEach{$i \in \idxs$}{
		\eIf{$i = N-1$}{
			$x^* ~:=~ x~+~\argmax\limits_{\delta \in \mathbb{B}_p(0,\epsilon)}~\left(\celoss(\onehot(\lab{\certify{F}^\id{i}}(x+\delta)),~\onehot(y))~+~\sum\limits_{k<i} (1 - \indicator\left[\cer{\certify{F}^\id{k}}(x+\delta)\right])\right)$\algnote\;
			%\footnotemark{\label{alg:note}}\;
		}{
			$x^* ~:=~ x~+~\argmax\limits_{\delta \in \mathbb{B}_p(0,\epsilon)}~\left(\indicator\left[\cer{\certify{F}^\id{i}}(x+\delta)\right]~+~\sum\limits_{k<i} (1 - \indicator\left[\cer{\certify{F}^\id{k}}(x+\delta)\right])\right)$\algnote\;
			%\textsuperscript{\ref{alg:note}}\;
		}
		\If{$\lab{\certify{F}}(x^*) \neq y$}{
			\textbf{return} $x^*$\;
		}
	}
	\textbf{return} $x$\;
}
\caption{Cascade Attack (CasA)}
\label{alg:attack}
\end{algorithm2e}
\footnotetext{In our implementation, we use a surrogate version of this objective (see Section~\ref{sec:attack}).}

Section~\ref{sec:cascade-unsoundness} shows that a cascading ensemble does not provide a robustness guarantee.
%that it claims to have. In this section, 
We further show here how one can attack the cascading ensemble and find an adversarial example within the $\epsilon$-ball centered at the input $x$.

\noindent\textbf{Overview of Attack.} Algorithm~\ref{alg:attack} describes the attack algorithm, \emph{cascade attack} (CasA), inspired by the proof of Theorem~\ref{thm:cascading_sound}. Given an input $x$, the goal of the algorithm is to find an input $x'$ in the $\epsilon$-ball centered at $x$ such that the predictions of the cascade at $x$ and $x'$ are different. 
% Since the cascade certifier is unsound, our algorithm is intended to be applicable even at inputs where the cascade declares itself to be certified robust.
%The intuitive idea behind the  algorithm, inspired by the construction showing the unsoundness of the cascade (Proof of Theorem~\ref{thm:cascading_sound}), is to find an input $x'$ that forces the ensemble to use a different constituent model for prediction.
% The design of the algorithm is inspired by the construction showing the unsoundness of the cascade (Proof of Theorem~\ref{thm:cascading_sound}).
The inputs to the algorithm are an ensemble $\certify{F}$, its constituent classifiers $\certify{F}^\id{0}$, $\ldots$,  $\certify{F}^\id{N-1}$,
the input $x$ to be attacked, and the attack distance bound $\epsilon$ as well as distance metric $\ell_p$. The algorithm either returns a successful adversarial input $x'$ such that $||x-x'||_p \leq \epsilon$ and $\lab{\certify{F}}(x) \neq \lab{\certify{F}}(x')$ or it returns the original input $x$ if no adversarial input was found. We use the following notations: $\celoss$ stands for cross-entropy loss, $\onehot$ is the one-hot encoding function, $\indicator$ is the indicator function, and $\mathbb{B}_p(0,\epsilon)$ is the $\ell_p$-ball of radius $\epsilon$ centered at $0$.
 
\noindent\textbf{Preparing Targets.} CasA gets the label $y$ predicted by the ensemble $\certify{F}$ at input $x$ (line 2) to select the constituent models it may attack. The attacker is only interested in a constituent model (by remembering its index $i$) if it predicts a label other than $y$ at $x$ or it is the last one (line 3). We are not interested in attacking a model $\certify{F}^\id{j}$ that predicts $y$ at $x$ because such an effort is bound to fail. $\certify{F}^\id{j}$ is still sound even though the ensemble is not; therefore, no point $x'$ assigned a label other than $y$ by $\certify{F}^\id{j}$ is such that it is both less than $\epsilon$-away from $x$ and $\certify{F}^\id{j}$ is also certifiably robust at $x'$ (the second condition is necessary for $\certify{F}^\id{j}$ to be used for prediction at $x'$ by the ensemble).
%at least $\epsilon$-away from the decision boundary of $\certify{F}^\id{j}$ (to be used as the return of the ensemble).
%% This step is motivated by the observation that, for a successful adversarial input $x'$, the constituent model used for prediction at $x'$ predicts a label other than $y$ and is also certified robust at $x'$ (except for the case when the last model $\certify{F}^\id{N-1}$ is used for prediction). For constituents such that $\lab{\certify{F}^\id{i}}(x) = y$, there cannot be any point within the $\epsilon$-ball centered at $x$ where $\certify{F}^\id{i}$ predicts a different label and is certified robust, and therefore, we filter out such models. 
%Notice the last model $\certify{F}^\id{N-1}$ is an exception and always remembered, i.e. $\idxs$ always includes $N-1$ because $x'$ is an adversarial point to the entire ensemble if the predictions of $x$ and $x'$ are both made by $\certify{F}^\id{N-1}$, i.e. all previous models fail to certify, and $\lab{\certify{F}^\id{N-1}}(x') \neq y$. It does not matter whether $x'$ is $\epsilon$-away from the boundary of $\certify{F}^\id{N-1}$ or not.
%% As an ensemble must return the label of the last model, regardless of its prediction being certifiably robust or not, if all previous ones fail to certify. An 
%% if used as the return of the ensemble at $x'$, the last model only needs to predict a different label but need not to be certified for $x'$ to be a valid adversarial point.
However, the last model $\certify{F}^\id{N-1}$ is an exception and always remembered, i.e. $\idxs$ always includes $N-1$. The reason is that, given an input $x'$, if all models $\certify{F}^\id{i};i<N-1$ fail to be certifiably robust at $x'$, the ensemble uses $\certify{F}^\id{N-1}$ for prediction at $x'$ irrespective of whether $\certify{F}^\id{N-1}$ is itself certifiably robust at $x'$ or not.
 
\noindent\textbf{Attacker's Steps.} For each model index in $\idxs$, we try to find an adversarial example (lines 4-10). An attacker stops as soon as they find a valid adversarial example (lines 9-10). Lines 6 and 8 describe the objective an attacker minimizes to find the adversarial examples. If index $i \neq N-1$, the attacker optimizes $\delta$ such that, at input $x+\delta$, the model $\certify{F}^\id{i}$ is certified robust whereas all other models $\certify{F}^\id{k};k<i$ are not certified robust. This ensures that model $\certify{F}^\id{i}$ is used for prediction at input $x+\delta$ as it is certifiably robust at $x$.  If index $i = N-1$, we still require that all models $\certify{F}^\id{k};k<i$ are not certified robust at $x+\delta$. But instead of requiring that $\certify{F}^\id{i}$ is certified robust at $x+\delta$, we only require that the predicted label at $x+\delta$ be different from $y$. We solve the optimization problems on lines 6 and 8 using projected gradient descent (PGD)~\citep{madry18}.

\noindent\textbf{Surrogate Objectives.} For cases when the certification procedure, i.e. $\indicator[\cer{\certify{F}^\id{i}}(x+\delta)]$, is not differentiable or too expensive to run multiple times, we provide the following cheap surrogate replacements. The intuition underlying the surrogate versions is that, given a model, the distance to the decision boundary from an input is correlated with the margin between the top logit scores of the model at that input. 
%Such a correlation is commonly observed for models trained in an adversarial manner\todo{add citation?}. 
% Therefore, for the problem $\argmax_{\delta \in \mathbb{B}_p(0,\epsilon)} \indicator\left[\cer{\certify{F}^\id{i}}(x+\delta)\right]$, where the goal is to find an input $x+\delta$ such that $\certify{F}^\id{i}$ is certified robust at $x+\delta$ (i.e., far away from the decision boundary), 
For the problem $\argmax_{\delta \in \mathbb{B}_p(0,\epsilon)} \indicator[\cer{\certify{F}^\id{i}}(x+\delta)]$, we try to increase the logit score associated with the desired prediction as much as possible. Then, a surrogate version of the problem is as follows (where $\logit{\certify{F}^\id{i}}$ represents the logit scores produced by model $\certify{F}^\id{i}$):
\begin{align}\label{eq:attack_surrogate1}
	\argmax_{\delta \in \mathbb{B}_p(0,\epsilon)} -\celoss(\logit{\certify{F}^\id{i}}(x+\delta),~\onehot(\lab{\certify{F}^\id{i}}(x)))
\end{align}
 For the problem $\argmax_{\delta \in \mathbb{B}_p(0,\epsilon)} \sum_{k<i} (1 - \indicator[\cer{\certify{F}^\id{k}}(x+\delta)])$, we want the input $x+\delta$ to be as close as possible to the decision boundaries for each of the models by $\certify{F}^\id{k}, k < i$ so that the robustness certifications will fail. The specific predictions $\lab{\certify{F}^\id{k}}(x+\delta)$ of these models do not matter. Towards that end, we aim to make the margin between the logit scores of any model $\certify{F}^\id{k}$ be as small as possible. This leads to the following surrogate problem (where $\unifdist$ is a discrete uniform distribution):
\begin{align}\label{eq:attack_surrogate2}
	\argmax_{\delta \in \mathbb{B}_p(0,\epsilon)} -\sum_{k<i} \celoss(\logit{\certify{F}^\id{k}}(x+\delta),~\unifdist)
\end{align}

% !TEX root=./main.tex
\section{Empirical Evaluation}
\label{sec:eval}
%\footnotetext{Uses \emph{exact} mode of dual networks based local robustness certification during training.} 
The goal of our empirical evaluation is to demonstrate the extent to which the unsoundness of the cascading ensembles manifests in practice and can be exploited by an adversary, i.e. CasA.
% \paragraph{Experimental Setup.}
% %The goal of our empirical evaluation is to measure the extent to which the unsoundness of the cascading ensembles distorts the ensemble $\cra$.
For our measurements, we use the $\ell_\infty$ and $\ell_2$ robust cascading ensembles constructed by \citet{wong2018scaling} for MNIST~\citep{lecun1998gradient} and CIFAR-10~\citep{krizhevsky09} datasets. 
The constituent classifiers in these cascades use a local robustness certification mechanism based on dual networks~\citep{wong2018scaling}. 
Each cascade includes between 2-7 certifiable classifiers with the same architecture (except for the $\ell_\infty$ robust, CIFAR-10 Resnet cascades that include only a single constituent model, and are hence not considered in our evaluation).
% Moreover, the constituent classifiers are trained in a sequential manner such that the later constituents of the cascade are trained only on those 
% examples that the previous constituents of the cascade cannot certify. 
The training code and all the constituent models in the ensembles are made available by \citet{wong2018scaling} in a public repository~\citep{wongcode}.
%\footnote{\url{https://github.com/locuslab/convex_adversarial}}

% We attack the cascading ensemble with CasA. Certifying with dual networks~\cite{wong2018scaling} is differentiable but extremely expensive. To run the attack more efficiently, we use the surrogate replacements in Section~\ref{sec:surrogate} and take 100 steps using PGD (other hyper-parameters to follow in Appendix~\ref{sec:hparams of tabel 1}) to empirically minimize the objectives. 

We report the certified robust accuracy (\cra) and standard accuracy (\acc) for the cascading ensemble as well as the single best constituent model in the ensemble. While the certifier for a single model is sound,  the ensemble certifier is unsound and the reported ensemble $\cra$ is an over-estimate. We therefore measure the empirical robustness of the ensemble under CasA. Certifying with dual networks~\citep{wong2018scaling} is differentiable but extremely expensive. To run the attack more efficiently, we use the surrogate replacements in Section~\ref{sec:attack} and take 100 steps using PGD (other hyper-parameters to follow in Appendix~\ref{sec:hparams of tabel 1}) to empirically minimize the objectives. After the attack, we report the false positive rate (\fpr), i.e. \% of test inputs for which an adversarial example is found within the $\epsilon$-ball, and the empirical robust accuracy (\era), i.e. \% of test set where the cascade is empirically robust (i.e., our attack failed). All our experiments were run on a pair of NVIDIA TITAN RTX GPUs with 24 GB of RAM each, and a 4.2GHz Intel Core i7-7700K with 64 GB of RAM.

\begin{table}[!t]
	\caption{Results on models pre-trained by \citet{wong2018scaling} for $\ell_\infty$ (top) and $\ell_2$ (bottom) robustness. \cra: \% of test set where model is certified robust and accurate. \acc: \% of test set where model is accurate. \fpr: among all test inputs where cascade is certified robust and accurate, \% of inputs for which an adversarial example is found within the $\epsilon$-ball using our ensemble attack (i.e., false positive rate). \era: \% of test set where the cascading is empirically robust (i.e., our attack failed) and accurate (\era~of a single model is always greater or equal to its~\cra~because of \emph{soundness} and therefore not included). The unsoundness of cascade certification is shown by the high false positive rates (\fpr).}
    \label{tab:ensemble_attack}
    \centering
    \resizebox{0.9\textwidth}{!}{
    \begin{tabular}{ c c c | c c | c c c c }
    \toprule
    \toprule
    {$\ell_\infty$} & {} & {} & \multicolumn{2}{c}{Single Model} & \multicolumn{4}{|c}{Cascading Ensemble} \\
	    Dataset & Model & $\epsilon$ & \cra(\%) & \acc(\%) & unsound \cra(\%) & \fpr(\%)&  \acc(\%) & \era(\%) \\
    \midrule
	    MNIST & Small, Exact%\tabnote 
	    & 0.1 & 95.54 & 98.96 & 96.33 & 88.71&96.62 &11.17 \\
    MNIST & Small & 0.1 & 94.94 & 98.79 & 96.07 & 81.93&96.24 &17.51\\
    MNIST & Large & 0.1 & 95.55 & 98.81 & 96.27 & 86.37&96.42  &13.27\\
    \midrule
    MNIST & Small & 0.3 & 56.21 & 85.23 & 65.41 & 88.87&65.80  &7.67\\
    MNIST & Large & 0.3 & 58.02 & 88.84 & 65.50 &85.27 &65.50  &9.65\\
    \midrule
    CIFAR10 & Small & 2/255 & 46.43 & 60.86 & 56.65 & 11.51&56.65  &50.13\\
    CIFAR10 & Large & 2/255 & 52.65 & 67.70 & 64.87 & 10.47&65.13   & 58.15\\

    \midrule
    CIFAR10 & Small & 8/255 & 20.58 & 27.60 &28.32 &16.00 &28.32 & 23.79\\
    CIFAR10 & Large & 8/255 & 16.04 & 19.01 & 20.83 &17.67&20.83   &17.15\\
    \bottomrule\\
    \end{tabular}}
    \resizebox{0.9\textwidth}{!}{
    \begin{tabular}{ c c c | c c | c c c c }
    \toprule
    \toprule
    {$\ell_2$} & {} & {} & \multicolumn{2}{c}{Single Model} & \multicolumn{4}{|c}{Cascading} \\
	    Dataset & Model & $\epsilon$ & \cra(\%) & \acc(\%) & unsound \cra(\%) & \fpr(\%)&  \acc(\%) & \era(\%) \\
    \midrule
	    MNIST & Small, Exact%\tabnote
	    & 1.58 & 43.52 & 88.14 & 75.58 & 44.72&80.43& 43.46\\
    MNIST & Small & 1.58 & 43.34 & 87.73 & 74.66 & 40.93&79.07 &45.73\\
    MNIST & Large & 1.58 & 43.96 & 88.39 & 74.50 & 51.95&74.99  &35.81\\
    \midrule
    CIFAR10 & Small & 36/255 & 46.05 & 54.39 & 49.89 & 3.61 & 51.37  &49.27\\
    CIFAR10 & Large & 36/255 & 50.26 & 60.14 & 58.72 & 2.70 & 58.76  &57.17 \\
    CIFAR10 & Resnet & 36/255 & 51.65 & 60.70 &58.65 &3.41 & 58.69 & 56.68\\
    \bottomrule\\
    \end{tabular}}
\end{table}

Table~\ref{tab:ensemble_attack} shows the results for $\ell_\infty$ robustness (top) and $\ell_2$ robustness (bottom).
Each row in the table represents a specific combination of dataset (MNIST or CIFAR-10), architecture (Small or Large convolutional networks), and $\epsilon$ value used for local robustness certification. 
The structure of the table is the same as Tables 2 and 4 in~\citep{wong2018scaling}, except we add the columns reporting $\fpr$ and $\era$.

\textbf{Summary of Results.}
We see from Table~\ref{tab:ensemble_attack} that, irrespective of the dataset, model, $\epsilon$ value, or $\ell_p$ metric under consideration, our attack can find false positives, with false positive rates (\fpr) as high as 88.87\%.
In other words, there always exist test inputs where the ensemble is accurate and declares itself to be certified robust, but our attack is able to find an adversarial example. 
This result demonstrates that the unsoundness of the cascading ensemble certification mechanism is not just a problem in theory but it can be exploited by adversaries in practice. 
More strikingly, the empirical robust accuracy (\era) of the ensemble is often significantly lower than the certified robust accuracy (\cra) of the best constituent model. 
Since the $\era$ of a model is an upper-bound of its $\cra$, the actual $\cra$ of the ensemble can be no larger than the reported $\era$. This result shows that the use of a cascading ensemble can actually hurt the robustness instead of improving it. 
%Additionally, it demonstrates the effectiveness of our attack algorithm.

\textbf{Attack Efficiency.} In Table~\ref{tab:the efficiency of the attack}, we compare the attack results of CasA using the original objectives, i.e. dual networks~\citep{wong2018scaling}, and using surrogate replacements. Because the ensemble on MNIST contains more constituent models, it uses more memory with dual networks compared to CIFAR10. Our report of run time and memory usage shows that using surrogate replacements allows us to run attacks with larger batch size, less memory and time to reach the same level of performance.

\begin{table}[!t]
	\caption{Run time and peak memory usage of CasA. Results are reported on one Titan RTX.}
    \label{tab:the efficiency of the attack}
    \centering

    \resizebox{\textwidth}{!}{
    \begin{tabular}{ c c c c c c c c c c}
    \toprule
    \toprule

	Dataset & Model & \# of Models & $\epsilon$ & Objective &unsound \cra(\%)  & \era(\%) & Batch Size & Mem./Batch (MB) & Time/Batch (Min) \\
    \midrule
	MNIST & Small & 7 & 0.3 ($\ell_\infty$) & dual networks & 65.41 & 4.93 & 32 & 12565 & 4.34 \\
	MNIST & Small & 7 & 0.3 ($\ell_\infty$) & our surrogates & 65.41 & 7.67 & 32 & 4548 & 0.06\\
    \midrule
	CIFAR10 & Small & 5 & 2/255 ($\ell_\infty$) & dual networks & 28.32 & 23.79 & 32 & 10124 & 2.43 \\
	CIFAR10 & Small &5 & 2/255 ($\ell_\infty$) & our surrogates & 28.32  &22.92 & 32 & 4898 & 0.06\\
    \bottomrule\\
    \end{tabular}}
\end{table}

% !TEX root=./main.tex

\section{A Query-Access, Soundness-Preserving Ensembler}
\label{sec:defense}

We present a query-access, soundness-preserving ensembler based on weighted voting in this section.
Voting is a natural choice for the design of a query-access ensembler but ensuring that the ensembler is soundness-preserving can be subtle.
Section~\ref{sec:voting_ensemble} defines our ensembler and proves that it is soundness-preserving.
%We also present a thought experiment demonstrating that a weighted voting ensemble can significantly improve upon the $\cra$ of its constituent models.
In Section~\ref{sec:thought}, we present a thought experiment demonstrating that it is possible for a voting ensemble to significantly improve upon the $\cra$ of its constituent models.
Appendix~\ref{sec:learn_weights} describes our algorithm for learning the weights to be used in weighted voting, and Appendix~\ref{sec:more_eval} presents initial empirical results with the voting ensemble.
Our results show that improving the $\cra$ via a voting ensemble can be difficult on realistic datasets since it requires the ensemble to demonstrate a suitable balance between diversity and similarity of its constituents, but we believe that this is a fruitful direction for future research.

\subsection{Weighted Voting Ensemble}
\label{sec:voting_ensemble}
% OLD NOTATION

% \begin{definition}[Voting Ensemble]\label{def:voting}
%     Let $\certify{F}^\id{0}, \certify{F}^\id{1}, ..., \certify{F}^\id{N-1}$ be $N$ models with certifications and $N$ is odd. A Voting Ensemble $\vote: \mathbb{R}^n \rightarrow \ClassSet$ and its corresponding certification $\certify{\vote}: \mathbb{R}^n \rightarrow \{\notrobust, \robust\}$ of these models are 
%     \begin{align}
%         \vote(x) = \arg\max_j \sum^{N-1}_{k=0} \mathbb{I}[F^\id{k}(x) = j], \vote(x) = \begin{cases} 
%             1 & \text{if } v(x, \vote(x)) > \Notrobust + \max_{j \neq \vote(x)}v(x, j)  \\
%             0 & \text{otherwise}
%         \end{cases}
%     \end{align} where $v(x, j) = \sum^{N-1}_{k=0} \certify{F}^\id{k}(x) * \mathbb{I}[(F^\id{k}(x) = j)]$ and $\Notrobust =  \sum^{N-1}_{k=0} (1-\certify{F}^\id{k})$.
% \end{definition}

% \begin{theorem}
%     If $\certify{F}^\id{0}, \certify{F}^\id{1}, ..., \certify{F}^\id{N-1}$ are sound, $\certify{\vote}$ is sound. Namely, 
%     \begin{align}
%         \forall x \in \mathbb{R}^d, \certify{\vote}(x) = \robust  \Longrightarrow \vote \text{ is } \epsilon\text{-locally robust at } x.
%     \end{align}
% \end{theorem}

%- Introduce intuition of voting ensembles here.
Voting ensembles run a vote amongst their constituent classifiers to make a prediction. 
In the simplest case, each constituent has a single vote that gets assigned to their predicted label. The label with the maximum number of votes is chosen as the ensemble's prediction. 
More generally, weighted voting allows a single constituent to be allocated more than one vote. The decimal number of votes allocated to each constituent is referred to as its \emph{weight}. 
For simplicity, we assume that weights of the constituents in an ensemble are normalized so that they sum up to 1. 
We use weighted voting to not only choose the ensemble's prediction but to also decide its certification outcome. The interaction between voting and certification is subtle and needs careful design to ensure that the certification procedure is sound.

\paragraph{Extra Notations.} 
%Let $v_x^w(j, c)$ denote the number of votes on input $x$ for label $j \in \ClassSet$ with certification outcome $c \in \{0,1\}$, assuming the certifiable classifiers in the ensemble are assigned weights from $w \in [0,1]^N$.
Let $v_x^w(j, c)$ denote the total number of votes allocated to certifiable classifiers, $\certify{F}^\id{i}$, in the ensemble that output $(j, c)$.
More formally, for an input $x$, label $j \in \ClassSet$, certification outcome $c \in \{0,1\}$,weight $w \in [0,1]^N$, and a set of constituent certifiable classifiers, $\certify{F}^\id{0}, \ldots, \certify{F}^\id{N-1}$,
let $v_x^w(j, c) := \sum_{i = 0}^{N-1} w_i * \indicator[\certify{F}^\id{i}(x) = (j, c)]$.
% given according to Equation~\ref{eq:vote_function}.
% \begin{equation}
% \label{eq:vote_function}
% v_x^w(j, c) := \sum_{i = 0}^{N-1} w_i * \indicator\left[\certify{F}^\id{i}(x) = (j, c)\right]
% \end{equation}
We find it useful to use $v_x^w(*, c)$ to denote the number of votes for any class with certificate, $c$; i.e., $v_x^w(*, c) = \sum_{j\in\ClassSet} v_x^w(j, x)$.
Likewise, we will use $v_x^w(j)$ to denote the number of votes for class $j$ regardless of certificate; i.e., $v_x^w(j) = v_x^w(j, 0) + v_x^w(j, 1)$. 

%Using $v_x^w$, we define a \emph{Weighted Voting Ensembler} according to Def~\ref{def:voting_ensembler}.

\begin{definition}[Weighted Voting Ensembler]\label{def:voting_ensembler}
    Let $\certify{F}^\id{1}, ..., \certify{F}^\id{N}$ be $N$ certifiable classifiers. 
    A weighted voting ensembler, $\ensembler_\vote^w: \mathbb{\certify{F}}^N \rightarrow \mathbb{\certify{F}}$ is defined as follows
    $$
	\certify{F}(x) := \ensembler_\vote^w(\certify{F}^\id{0}, \certify{F}^\id{1}, \ldots, \certify{F}^\id{N-1})(x) :=
	\big(\lab{\certify{F}}(x), \cer{\certify{F}}(x)\big), \quad	\lab{\certify{F}}(x) := \argmax_j\Big\{v_x^w(j)\Big\}\footnote{In case of a tie, we assume that the label corresponding to the logit with the lowest index is returned.}
    $$
    $$ \text{where  }
	\cer{\certify{F}}(x) := \begin{cases}
		1 & \text{if $\forall j \neq \lab{\certify{F}}(x)$ ~.~ $v_x^w(\lab{\certify{F}}(x), 1) > v_x^w(*, 0) + v_x^w(j, 1)$} \\
    0 & \text{otherwise}
    \end{cases}
    $$
\end{definition}

%Like the cascading ensembler, the weighted voting ensembler is also query-access (Thm.~\ref{thm:voting_qa}).
%\begin{theorem}
%\label{thm:voting_qa}
%Weighted voting ensembler $\ensembler_\vote^w$ is query-access.
%\end{theorem}
%\begin{proof}
%	See Appendix~\ref{sec:proofs} in supplemental material.
%\end{proof}

The prediction of the weighted voting ensemble is the label receiving the maximum number of votes regardless of the certificate.
However, for the certification outcome, the ensemble has to consider the certificates of the constituent models. 
The ensemble should be certified robust only if its prediction outcome, i.e., the label receiving the maximum number of votes (regardless of the certificate), can be guaranteed to not change in an $\epsilon$-ball.
The condition under which $\cer{\certify{F}}(x) = 1$ ensures this is the case, and allows us to prove that weighted voting ensemblers are soundness-preserving (Theorem~\ref{thm:voting_sound}).
A key observation underlying the condition is that only constituent classifiers that are not certified robust at the current input can change their predicted label in the $\epsilon$-ball, and, in the worst case, transfer all their votes ($v_x^w(*, 0)$) to the label with the second highest number of votes at $x$. We believe that our proof of soundness-preservation is of independent interest.
We also note that weighted-voting ensemblers are query-access (formal proofs in Appendix~\ref{sec:proofs}).

\begin{theorem}
\label{thm:voting_sound}
The weighted voting ensembler $\ensembler_\vote^w$ is soundness-preserving.
\end{theorem}

%We define a \emph{uniform voting ensembler} as a special case of weighted voting ensemblers where each constituent model gets the same weight.
\begin{definition}[Uniform Voting Ensembler]\label{def:unif_voting_ensembler}
    Let $\certify{F}^\id{0}, ..., \certify{F}^\id{N-1}$ be $N$ certifiable classifiers. 
	The uniform voting ensembler, $\ensembler_\uvote: \mathbb{\certify{F}}^N \rightarrow \mathbb{\certify{F}}$ is a weighted voting ensembler that assigns equal weights to each classifier, i.e. $\ensembler_\uvote = \ensembler_\vote^w$ where $\forall i \in \{0,\ldots,N-1\}. w_i = 1/N$.
\end{definition}

\subsection{Effectiveness of Voting: A Thought Experiment}
\label{sec:thought}
%Our evaluation suggests that improving upon the $\cra$ of the constituent models of a sound, query-access ensemble is difficult in practice.
Voting ensembles require the constituents to strike the right balance between diversity and similarity to be effective.
In other words, while the constituents should be accurate and robust in different regions of the input space (diversity), these regions should also have some overlap (similarity).
We conduct a thought experiment using a simple hypothetical example (Example.~\ref{ex:voting_improve}) where such a balance is struck.
The existence of this example provides evidence and hope that voting ensembles can improve the $\cra$.% of their constituents. 
We present the example informally here. The detailed, rigorous argument is in Appendix~\ref{sec:proofs}.

\begin{example}  
\label{ex:voting_improve}
%	Weighted voting ensembler $\ensembler_\vote^w$ can improve certified robust accuracy. 
Assume that we have a uniform voting ensemble $\certify{F}$ with three constituent classifiers $\certify{F}^\id{0}$, $\certify{F}^\id{1}$, and $\certify{F}^\id{2}$. 
Assume that on a given dataset $S_k$ with 100 samples, each of the constituent classifiers has $\cra$ equal to 0.5.
	Let's say that the samples in $S_k$ are ordered such that $\certify{F}^\id{0}$ is accurate and robust on the first 50 samples (i.e., samples 0-49),
$\certify{F}^\id{1}$ is accurate and robust on samples 25-74, and $\certify{F}^\id{2}$ on samples 0-24 and 50-74.
Then, for each of the first 75 samples, two out of three constituents in the ensemble are accurate and robust. 
	Therefore, by Def.~\ref{def:voting_ensembler}, the ensemble $\certify{F}$ is accurate and robust on samples 0-74, and has $\cra$ equal to 0.75.
\end{example}

% !TEX root=./main.tex

\section{Related Work}
\label{sec:related}
%- Ensembling, boosting, and bagging in general \\
Ensembling is a well-known approach for improving the accuracy of models as long as the constituent models are suitably diverse \citep{dietterich2000}. 
%Aproaches like bagging and boosting are theoretically guaranteed to produce ensembles with a better accuracy than their constituent models.
%More recently, \emph{deep ensembles}, i.e. ensembles of deep neural networks, that rely only on the randomness of initialization and of stochastic gradient descent over non-convex objectives, have been empirically shown to improve model accuracy.
In recent years, with the growing focus on \emph{robust accuracy} as a metric of model quality, a number of ensembling techniques have been proposed for improving this metric.
Depending on whether an ensemble is query-access or not (i.e., does not or does require access to the internal details of the constituent models for prediction and certification), it can be classified as a \emph{white-box} or a \emph{black-box} ensemble. The modularity of black-box ensembles is attractive as the constituent classifiers can each be from a different model family (i.e., neural networks, decision trees, support vector machines, etc.) and each use a different mechanism for robustness certification. 
The constituents of white-box ensembles, on the other hand, tend to be from the same model family but this provides the benefit of tuning the ensembling strategy to the model family being used. %and certification mechanism being used.

\textbf{White-box ensembles.} 
%- Certifiably robust averaging ensembles \cite{yang2022on,zhang2019enhancing,liu2020enhancing,abernethy2021multiclass}
%- Empirically robust averaging ensembles \cite{pang2019improving, yang2020dverge, kariyappa2019improving} \\
Several works \citep{yang2022on,zhang2019enhancing,liu2020enhancing} present certifiable ensembles where the ensemble logits are calculated by averaging the corresponding logits of the constituent classifiers. 
Needing access to the logits of the constituent classifiers, and not just their predictions, is one aspect that makes these ensembles white-box. 
More importantly, the approaches used by these ensembles for local robustness certification are also in violation of our definition of query-access ensembles (Def.~\ref{def:qa_ensembler}). For instance, randomized smoothing~\citep{cohen19} is used in
\citep{yang2022on,liu2020enhancing} to certify the ensemble, which requires evaluating the constituent models on a large number of inputs for each prediction, and not just one input. 
%This violates the requirement that a query-access ensemble should evaluate the constituent models only on a single input.
Other approaches \citep{zhang2019enhancing} use interval bound propagation (IBP)~\citep{gowal2018effectiveness,zhang18efficient} to certify the ensemble. Calculating the interval bounds requires access to the architecture and weights of each of the constituent models, violating the requirements of a query-access ensemble.
A number of white-box ensembling techniques \citep{pang2019improving, yang2020dverge, kariyappa2019improving, Sen2020EMPIR, pmlr-v162-zhang22aj} only aim to improve \emph{empirical} robust accuracy, i.e., these ensembles do not provide a robustness certificate. 
As before, the ensemble logits are calculated by averaging the corresponding logits of the constituent models. 
These approaches differ from each other in the training interventions used to promote diversity in the constituent models.

\textbf{Black-box ensembles.}
%- Incorrect certifiably robust cascade (voting-based) ensembles \cite{wong2018scaling, blum2022boosting} \\
%- Empirically robust voting-based ensemble \cite{cheng2020voting}
Cascading ensembles \citep{wong2018scaling, blum2022boosting} are the most popular example of certifiably robust black-box ensembles. 
While \citet{wong2018scaling} empirically evaluate their cascading ensemble, the results of \citet{blum2022boosting} are purely theoretical. However, as we show in this work, the certification mechanism used by cascading ensembles is unsound.
\citet{cheng2020voting, Sen2020EMPIR} present a black-box voting ensemble but, unlike our voting ensemble, their ensemble does not provide robustness certificates. Nevertheless, they are able to show improvements in the empirical robust accuracy with the voting ensemble.
%Their empirical results, in combination with our hypothetical example showing that voting ensembles can improve $\cra$, give us hope that voting ensembles can also improve the $\cra$ in practice.

% !TEX root=./main.tex

\section{Conclusion}
\label{sec:conclusion}
In this paper, we showed that the local robustness certification mechanism used by cascading ensembles is unsound.
%, i.e., when a cascading ensemble is certified as locally robust at an input $x$, there can be inputs $x'$ in the $\epsilon$-ball centered at $x$, such that the cascade's prediction at $x'$ is different from $x$.
As a consequence, %the use of cascading ensembles in  scenarios that require local robustness guarantees is unsafe, and 
existing empirical results that report the certified robust accuracies (\cra) of cascading ensembles are not valid.
Guided by our theoretical results, we designed an attack algorithm against cascading ensembles and demonstrated that their unsoundness can be easily exploited in practice. In fact, the performance of the ensembles is markedly worse than their single best constituent model.
Finally, we presented an alternate black-box ensembling mechanism based on weighted voting that we prove to be sound, and, via a thought experiment, showed that voting ensembles can significantly improve the $\cra$ if the constituent models have the right balance between diversity and similarity.
%On repeating the experiments performed by \citet{wong2018scaling} using our provably sound voting ensemble (instead of the cascading ensemble), we find no improvement in $\cra$ over the constituent models.
%However, there is reason to be optimistic about voting ensembles; via a thought experiment, we show that voting ensembles can significantly improve the $\cra$ if the constituent models have the right balance between diversity and similarity.

\section*{Ethics Statement}
\label{sec:ethics}
Our work sheds light on existing vulnerabilities in state-of-the-art certifiably robust neural classifiers.
The presented attacks can be used by malicious entities to adversarially attack deployed cascading ensembles of certifiably robust models. 
However, by putting this knowledge out in the public domain and making practitioners aware of the existence of the problem, we hope that precautions can be taken to protect existing systems.
Moreover, it highlights the need to harden future systems against such attacks.

\section*{Reproducibility Statement}
\label{sec:repro} 

To examine our theoretical results, the proof of Theorem~\ref{thm:cascading_sound} directly follows the body of the theorem in Section~\ref{sec:cascade-unsoundness} while the proof of Theorem~\ref{thm:voting_sound} is delayed to Appendix~\ref{sec:proofs} and~\ref{sec:permutation_ensembler}, together with the proofs of other theorems that only appear in the appendix, i.e. Theorem~\ref{thm:cascading_qa}, ~\ref{thm:voting_qa} (Appendix~\ref{sec:proofs}) and Theorem~\ref{thm:permutation} (Appendix~\ref{sec:permutation_ensembler}). All the datasets used in our work are publicly available with links in their corresponding reference.  Our experimental code is uploaded in the supplementary material with a detailed README file and weights of models to reproduce the results in Table~\ref{tab:ensemble_attack}, \ref{tab:the efficiency of the attack},~\ref{tab:l_inf_ensembles},~\ref{tab:l_2_ensembles}, ~\ref{tab:l_inf_new}, and~\ref{tab:l_2_new}. Moreover, hyper-parameters used in these table are also documented in Appendix~\ref{sec:hparams of tabel 1} and~\ref{sec:learn_weights}. The hardware information used in all experiments is reported in Section~\ref{sec:eval}.

\bibliography{references}

\begin{thebibliography}{31}
\providecommand{\natexlab}[1]{#1}
\providecommand{\url}[1]{\texttt{#1}}
\expandafter\ifx\csname urlstyle\endcsname\relax
  \providecommand{\doi}[1]{doi: #1}\else
  \providecommand{\doi}{doi: \begingroup \urlstyle{rm}\Url}\fi

\bibitem[Bauer \& Kohavi(1999)Bauer and Kohavi]{bauer1999empirical}
Eric Bauer and Ron Kohavi.
\newblock An empirical comparison of voting classification algorithms: Bagging,
  boosting, and variants.
\newblock \emph{Machine learning}, 36\penalty0 (1):\penalty0 105--139, 1999.

\bibitem[Blum et~al.(2022)Blum, Montasser, Shakhnarovich, and
  Zhang]{blum2022boosting}
Avrim Blum, Omar Montasser, Greg Shakhnarovich, and Hongyang Zhang.
\newblock Boosting barely robust learners: A new perspective on adversarial
  robustness.
\newblock \emph{arXiv preprint arXiv:2202.05920}, 2022.

\bibitem[Cohen et~al.(2019)Cohen, Rosenfeld, and Kolter]{cohen19}
Jeremy Cohen, Elan Rosenfeld, and Zico Kolter.
\newblock Certified adversarial robustness via randomized smoothing.
\newblock In Kamalika Chaudhuri and Ruslan Salakhutdinov (eds.),
  \emph{Proceedings of the 36th International Conference on Machine Learning},
  volume~97 of \emph{Proceedings of Machine Learning Research}, pp.\
  1310--1320. PMLR, 09--15 Jun 2019.
\newblock URL \url{https://proceedings.mlr.press/v97/cohen19c.html}.

\bibitem[Devvrit et~al.(2020)Devvrit, Cheng, Hsieh, and
  Dhillon]{cheng2020voting}
Devvrit, Minhao Cheng, Cho-Jui Hsieh, and Inderjit~S. Dhillon.
\newblock Voting based ensemble improves robustness of defensive models.
\newblock \emph{ArXiv}, abs/2011.14031, 2020.

\bibitem[Dietterich(2000)]{dietterich2000}
Thomas~G. Dietterich.
\newblock Ensemble methods in machine learning.
\newblock In \emph{Multiple Classifier Systems}, pp.\  1--15, Berlin,
  Heidelberg, 2000. Springer Berlin Heidelberg.
\newblock ISBN 978-3-540-45014-6.

\bibitem[Fromherz et~al.(2021)Fromherz, Leino, Fredrikson, Parno, and
  Păsăreanu]{fromherz20}
Aymeric Fromherz, Klas Leino, Matt Fredrikson, Bryan Parno, and Corina
  Păsăreanu.
\newblock Fast geometric projections for local robustness certification.
\newblock In \emph{International Conference on Learning Representations
  (ICLR)}, 2021.

\bibitem[Gilmer et~al.(2019)Gilmer, Ford, Carlini, and Cubuk]{gilmer19a}
Justin Gilmer, Nicolas Ford, Nicholas Carlini, and Ekin Cubuk.
\newblock Adversarial examples are a natural consequence of test error in
  noise.
\newblock In Kamalika Chaudhuri and Ruslan Salakhutdinov (eds.),
  \emph{Proceedings of the 36th International Conference on Machine Learning},
  volume~97 of \emph{Proceedings of Machine Learning Research}, pp.\
  2280--2289. PMLR, 09--15 Jun 2019.
\newblock URL \url{https://proceedings.mlr.press/v97/gilmer19a.html}.

\bibitem[Gowal et~al.(2018)Gowal, Dvijotham, Stanforth, Bunel, Qin, Uesato,
  Arandjelovic, Mann, and Kohli]{gowal2018effectiveness}
Sven Gowal, Krishnamurthy Dvijotham, Robert Stanforth, Rudy Bunel, Chongli Qin,
  Jonathan Uesato, Relja Arandjelovic, Timothy Mann, and Pushmeet Kohli.
\newblock On the effectiveness of interval bound propagation for training
  verifiably robust models.
\newblock \emph{arXiv preprint arXiv:1810.12715}, 2018.

\bibitem[Jordan et~al.(2019)Jordan, Lewis, and Dimakis]{jordan19provable}
Matt Jordan, Justin Lewis, and Alexandros~G Dimakis.
\newblock Provable certificates for adversarial examples: Fitting a ball in the
  union of polytopes.
\newblock In \emph{Advances in Neural Information Processing Systems},
  volume~32. Curran Associates, Inc., 2019.
\newblock URL
  \url{https://proceedings.neurips.cc/paper/2019/file/ae3f4c649fb55c2ee3ef4d1abdb79ce5-Paper.pdf}.

\bibitem[Kariyappa \& Qureshi(2019)Kariyappa and
  Qureshi]{kariyappa2019improving}
Sanjay Kariyappa and Moinuddin~K Qureshi.
\newblock Improving adversarial robustness of ensembles with diversity
  training.
\newblock \emph{arXiv preprint arXiv:1901.09981}, 2019.

\bibitem[Katz et~al.(2017)Katz, Barrett, Dill, Julian, and
  Kochenderfer]{katz17}
Guy Katz, Clark Barrett, David~L Dill, Kyle Julian, and Mykel~J Kochenderfer.
\newblock Reluplex: An efficient smt solver for verifying deep neural networks.
\newblock In \emph{International Conference on Computer Aided Verification},
  pp.\  97--117. Springer, 2017.

\bibitem[Krizhevsky(2009)]{krizhevsky09}
Alex Krizhevsky.
\newblock Learning multiple layers of features from tiny images.
\newblock Technical report, 2009.

\bibitem[LeCun et~al.(1998)LeCun, Bottou, Bengio, and
  Haffner]{lecun1998gradient}
Yann LeCun, L{\'e}on Bottou, Yoshua Bengio, and Patrick Haffner.
\newblock Gradient-based learning applied to document recognition.
\newblock \emph{Proceedings of the IEEE}, 86\penalty0 (11):\penalty0
  2278--2324, 1998.

\bibitem[Lee et~al.(2020)Lee, Lee, and Park]{lee20lipschitz}
Sungyoon Lee, Jaewook Lee, and Saerom Park.
\newblock Lipschitz-certifiable training with a tight outer bound.
\newblock In H.~Larochelle, M.~Ranzato, R.~Hadsell, M.F. Balcan, and H.~Lin
  (eds.), \emph{Advances in Neural Information Processing Systems}, volume~33,
  pp.\  16891--16902. Curran Associates, Inc., 2020.
\newblock URL
  \url{https://proceedings.neurips.cc/paper/2020/file/c46482dd5d39742f0bfd417b492d0e8e-Paper.pdf}.

\bibitem[Leino et~al.(2021)Leino, Wang, and Fredrikson]{leino21}
Klas Leino, Zifan Wang, and Matt Fredrikson.
\newblock Globally-robust neural networks.
\newblock In \emph{International Conference on Machine Learning (ICML)}, 2021.

\bibitem[Liu et~al.(2020)Liu, Feng, Wang, and Dong]{liu2020enhancing}
Chizhou Liu, Yunzhen Feng, Ranran Wang, and Bin Dong.
\newblock Enhancing certified robustness via smoothed weighted ensembling.
\newblock \emph{arXiv preprint arXiv:2005.09363}, 2020.

\bibitem[Madry et~al.(2018)Madry, Makelov, Schmidt, Tsipras, and
  Vladu]{madry18}
Aleksander Madry, Aleksandar Makelov, Ludwig Schmidt, Dimitris Tsipras, and
  Adrian Vladu.
\newblock Towards deep learning models resistant to adversarial attacks.
\newblock In \emph{International Conference on Learning Representations}, 2018.

\bibitem[Pang et~al.(2019)Pang, Xu, Du, Chen, and Zhu]{pang2019improving}
Tianyu Pang, Kun Xu, Chao Du, Ning Chen, and Jun Zhu.
\newblock Improving adversarial robustness via promoting ensemble diversity.
\newblock In \emph{International Conference on Machine Learning}, pp.\
  4970--4979. PMLR, 2019.

\bibitem[Raghunathan et~al.(2018)Raghunathan, Steinhardt, and
  Liang]{raghunathan2018certified}
Aditi Raghunathan, Jacob Steinhardt, and Percy Liang.
\newblock Certified defenses against adversarial examples.
\newblock In \emph{International Conference on Learning Representations}, 2018.
\newblock URL \url{https://openreview.net/forum?id=Bys4ob-Rb}.

\bibitem[Sen et~al.(2020)Sen, Ravindran, and Raghunathan]{Sen2020EMPIR}
Sanchari Sen, Balaraman Ravindran, and Anand Raghunathan.
\newblock Empir: Ensembles of mixed precision deep networks for increased
  robustness against adversarial attacks.
\newblock In \emph{International Conference on Learning Representations}, 2020.
\newblock URL \url{https://openreview.net/forum?id=HJem3yHKwH}.

\bibitem[Singh et~al.(2019)Singh, Gehr, P\"{u}schel, and Vechev]{singh19}
Gagandeep Singh, Timon Gehr, Markus P\"{u}schel, and Martin Vechev.
\newblock An abstract domain for certifying neural networks.
\newblock \emph{Proc. ACM Program. Lang.}, 3\penalty0 (POPL), January 2019.

\bibitem[Szegedy et~al.(2014)Szegedy, Zaremba, Sutskever, Bruna, Erhan,
  Goodfellow, and Fergus]{szegedy13}
Christian Szegedy, Wojciech Zaremba, Ilya Sutskever, Joan Bruna, Dumitru Erhan,
  Ian~J. Goodfellow, and Rob Fergus.
\newblock Intriguing properties of neural networks.
\newblock In Yoshua Bengio and Yann LeCun (eds.), \emph{2nd International
  Conference on Learning Representations, {ICLR} 2014, Banff, AB, Canada, April
  14-16, 2014, Conference Track Proceedings}, 2014.
\newblock URL \url{http://arxiv.org/abs/1312.6199}.

\bibitem[Weng et~al.(2018)Weng, Zhang, Chen, Song, Hsieh, Daniel, Boning, and
  Dhillon]{weng18}
Lily Weng, Huan Zhang, Hongge Chen, Zhao Song, Cho-Jui Hsieh, Luca Daniel,
  Duane Boning, and Inderjit Dhillon.
\newblock Towards fast computation of certified robustness for relu networks.
\newblock In \emph{International Conference on Machine Learning}, pp.\
  5276--5285. PMLR, 2018.

\bibitem[Wong \& Kolter()Wong and Kolter]{wongcode}
Eric Wong and J~Zico Kolter.
\newblock URL \url{https://github.com/locuslab/convex_adversarial}.

\bibitem[Wong \& Kolter(2018)Wong and Kolter]{wong2018provable}
Eric Wong and Zico Kolter.
\newblock Provable defenses against adversarial examples via the convex outer
  adversarial polytope.
\newblock In \emph{International Conference on Machine Learning}, pp.\
  5286--5295. PMLR, 2018.

\bibitem[Wong et~al.(2018)Wong, Schmidt, Metzen, and Kolter]{wong2018scaling}
Eric Wong, Frank Schmidt, Jan~Hendrik Metzen, and J~Zico Kolter.
\newblock Scaling provable adversarial defenses.
\newblock \emph{Advances in Neural Information Processing Systems}, 31, 2018.

\bibitem[Yang et~al.(2020)Yang, Zhang, Dong, Inkawhich, Gardner, Touchet,
  Wilkes, Berry, and Li]{yang2020dverge}
Huanrui Yang, Jingyang Zhang, Hongliang Dong, Nathan Inkawhich, Andrew Gardner,
  Andrew Touchet, Wesley Wilkes, Heath Berry, and Hai Li.
\newblock Dverge: diversifying vulnerabilities for enhanced robust generation
  of ensembles.
\newblock \emph{Advances in Neural Information Processing Systems},
  33:\penalty0 5505--5515, 2020.

\bibitem[Yang et~al.(2022)Yang, Li, Xu, Kailkhura, Xie, and Li]{yang2022on}
Zhuolin Yang, Linyi Li, Xiaojun Xu, Bhavya Kailkhura, Tao Xie, and Bo~Li.
\newblock On the certified robustness for ensemble models and beyond.
\newblock In \emph{International Conference on Learning Representations}, 2022.
\newblock URL \url{https://openreview.net/forum?id=tUa4REjGjTf}.

\bibitem[Zhang et~al.(2022)Zhang, Zhang, Courville, Bengio, Ravikumar, and
  Suggala]{pmlr-v162-zhang22aj}
Dinghuai Zhang, Hongyang Zhang, Aaron Courville, Yoshua Bengio, Pradeep
  Ravikumar, and Arun~Sai Suggala.
\newblock Building robust ensembles via margin boosting.
\newblock In Kamalika Chaudhuri, Stefanie Jegelka, Le~Song, Csaba Szepesvari,
  Gang Niu, and Sivan Sabato (eds.), \emph{Proceedings of the 39th
  International Conference on Machine Learning}, volume 162 of
  \emph{Proceedings of Machine Learning Research}, pp.\  26669--26692. PMLR,
  17--23 Jul 2022.
\newblock URL \url{https://proceedings.mlr.press/v162/zhang22aj.html}.

\bibitem[Zhang et~al.(2018)Zhang, Weng, Chen, Hsieh, and
  Daniel]{zhang18efficient}
Huan Zhang, Tsui-Wei Weng, Pin-Yu Chen, Cho-Jui Hsieh, and Luca Daniel.
\newblock Efficient neural network robustness certification with general
  activation functions.
\newblock In \emph{Proceedings of the 32nd International Conference on Neural
  Information Processing Systems}, NIPS'18, pp.\  4944–4953, Red Hook, NY,
  USA, 2018. Curran Associates Inc.

\bibitem[Zhang et~al.(2019)Zhang, Cheng, and Hsieh]{zhang2019enhancing}
Huan Zhang, Minhao Cheng, and Cho-Jui Hsieh.
\newblock Enhancing certifiable robustness via a deep model ensemble.
\newblock \emph{arXiv preprint arXiv:1910.14655}, 2019.

\end{thebibliography}
\bibliographystyle{iclr2023_conference}

\newpage
\appendix
%\begin{center}
%	\textbf{\Large Appendix: On the Perils of Cascading Robust Classifiers}
%\end{center}

\section{Proofs}
\label{sec:proofs}

\noindent
\textbf{Theorem~\ref{thm:voting_sound}.}
\textit{ The weighted voting ensembler $\ensembler_\vote^w$ is soundness-preserving.}

\begin{proof}
Let $\certify{F}^\id{0}, ..., \certify{F}^\id{N-1}$ be $N$ certifiable classifiers, which we assume are sound.
Let $\certify{F} := \ensembler_\vote^w(\certify{F}^\id{0}, ..., \certify{F}^\id{N-1})$; i.e., $\certify{F}$ is given by Definition~\ref{def:voting_ensembler}.

	Assume for the sake of contradiction $\exists x, x'$ s.t. $||x-x'|| \leq \epsilon$, $\certify{F}(x) = (j_1, 1)$, and $\lab{\certify{F}}(x') = j_2$ where $j_2 \neq j_1$.

Since $\certify{F}(x) = (j_1, 1)$, by Definition~\ref{def:voting_ensembler}., $\forall j \neq j_1$, $v_x^w(j_1, 1) > v_x^w(*, 0) + v_x^w(j, 1)$, and thus, in particular, Equation~\ref{eq:soundness_proof:ineq_1} holds.
\begin{equation}
\label{eq:soundness_proof:ineq_1}
v_x^w(j_1, 1) > v_x^w(*, 0) + v_x^w(j_2, 1)
\end{equation}

Consider the votes on $x'$. 
The models that contribute to $v_x^w(j_1, 1)$ are all locally robust\footnote{The models are known to be locally robust because they are \emph{sound} and their output matches $(*, 1)$.} at $x$, so each of these models must output the label $j_1$ on $x'$, which is at distance no greater than $\epsilon$ from $x$; thus Equation~\ref{eq:soundness_proof:ineq_2} holds.
\begin{equation}
\label{eq:soundness_proof:ineq_2}
v_{x'}^w(j_1) \geq v_x^w(j_1, 1)
\end{equation}

Conversely, only points that are non-robust at $x$ can change labels on $x'$, thus we obtain Equation~\ref{eq:soundness_proof:ineq_3}.
\begin{equation}
\label{eq:soundness_proof:ineq_3}
v_{x'}^w(j_2) \leq v_x^w(*,0) + v_x^w(j_2, 1)
\end{equation}

Putting things together we have
\begin{align*}
v_{x'}^w(j_1) &\geq v_x^w(j_1, 1) &\text{by (\ref{eq:soundness_proof:ineq_2})} \\
            &> v_x^w(*,0) + v_x^w(j_2, 1) &\text{by (\ref{eq:soundness_proof:ineq_1})} \\
            &\geq v_{x'}^w(j_2) &\text{by (\ref{eq:soundness_proof:ineq_3})}
\end{align*}

	Thus, since $v_{x'}^w(j_1) > v_{x'}^w(j_2)$, $\lab{\certify{F}}(x')$ cannot be $j_2$ \lightning.
\end{proof}

\noindent
\textbf{Example~\ref{ex:voting_improve}.}
%\textit{ Weighted voting ensembler $\ensembler_\vote^w$ can improve certified robust accuracy.}
{\itshape
We want to show that,
	$\exists \certify{F}^\id{0}, \certify{F}^\id{1}, \ldots, \certify{F}^\id{N-1} \in \mathbb{\certify{F}}, w \in [0,1]^N$ such that for $\certify{F} := \ensembler_\cas^w(\certify{F}^\id{0}, \certify{F}^\id{1}, \ldots, \certify{F}^\id{N-1}),$

	$$\exists S_k \subseteq \mathbb{R}^d \times \ClassSet.~\forall i \in \{0,\ldots,N-1\}.~\cra(\certify{F},S_k) > \cra(\certify{F}^\id{i},S_k)$$

Consider a weighted voting ensemble $\certify{F}$ constituted of certifiable classifiers $\certify{F}^\id{0}$, $\certify{F}^\id{1}$, and 
	$\certify{F}^\id{2}$ with weights $w = (\frac{1}{3},\frac{1}{3},\frac{1}{3})$, i.e., $\certify{F}$ is a uniform voting ensemble.

Suppose $k=100$, i.e., $S_k$ is a dataset with 100 samples. Moreover, lets say that $\cra(\certify{F}^\id{0},S_k) = \cra(\certify{F}^\id{1},S_k) = \cra(\certify{F}^\id{2},S_k) = 0.5$.
	Also, suppose that the samples in $S_k$ are arranged in a fixed sequence $S_k = (x_0,y_0), \ldots, (x_{k-1},y_{k-1})$ such that,
	\begin{align}
		\forall i \in [0,49].~\certify{F}^\id{0}(x_i) = (y_i,1) \label{eq:vot_im1}\\
		\forall i \in [25,74].~\certify{F}^\id{1}(x_i) = (y_i,1) \label{eq:vot_im2}\\
		\forall i \in [0,24] \cup [50,74].~\certify{F}^\id{2}(x_i) = (y_i,1) \label{eq:vot_im3}
	\end{align}
	where $[i,j]$ is the set of integers from $i$ to $j$, $i$ and $j$ included. 
	\ref{eq:vot_im1}, \ref{eq:vot_im2}, and \ref{eq:vot_im3} are consistent with the fact that certified robust accuracy of each model is 0.5.
	
	From \ref{eq:vot_im1}, \ref{eq:vot_im2}, \ref{eq:vot_im3}, and Definition~\ref{def:voting_ensembler},
	\begin{align}	
		\forall i \in [0,74].~\lab{\certify{F}}(x_i) = \argmax_j\Big\{v_{x_i}^w(j)\Big\} = y_i \label{eq:vot_im4}\\
		\forall i \in [0,74].~\forall j \neq y_i ~.~ v_{x_i}^w(y_i, 1) > v_{x_i}^w(*, 0) + v_{x_i}^w(j, 1) \label{eq:vot_im5}
	\end{align}

	From \ref{eq:vot_im5} and Definition~\ref{def:voting_ensembler},
	\begin{align}
		\forall i \in [0,74].  \cer{\certify{F}}(x_i) = 1 \label{eq:vot_im6}
	\end{align}

	From \ref{eq:vot_im4}, \ref{eq:vot_im6}, and Definition~\ref{def:cra},
	\begin{align}
		\cra(\certify{F},S_k) = 0.75 \label{eq:vot_im7}
	\end{align}
}

\begin{theorem}
\label{thm:cascading_qa}
The cascading ensembler $\ensembler_\cas$ is query-access.
\end{theorem}
\begin{proof}
	Let $\certify{F} := \ensembler_\cas(\certify{F}^\id{0}, \certify{F}^\id{1}, \ldots, \certify{F}^\id{N-1})$.
	Let $g^\id{0}, g^\id{1}, \ldots, g^\id{N-1} \in \ClassSet \times \{\notrobust, \robust\}$. We use $g_2^\id{j}$ to refer to the second element in the pair $g^\id{j}$.
	Define $G$ as follows,
	\begin{align*}
	    G(g^\id{0}, g^\id{1}, \ldots, g^\id{N-1}) := 
        \begin{cases} 
		g^\id{j} & \text{if } \exists j \leq N-1.~\cond(j)=1 \\
		g^\id{{N-1}} & \text{otherwise}
        \end{cases}
        \end{align*}
	where $\cond(j) := 1 \text{ if } (g_2^\id{j} = \robust) \wedge (\forall i < j, g_2^\id{i} = \notrobust)$ and 0 otherwise.

	Then, by Def.~\ref{def:cascading_ensembler}, $\certify{F}(x) = G(\certify{F}^\id{0}(x), \certify{F}^\id{1}(x), \ldots, \certify{F}^\id{N-1}(x))$.
	Then, by Def.~\ref{def:qa_ensembler}, cascading ensembler is query-access.
\end{proof}

\begin{theorem}
\label{thm:voting_qa}
The weighted voting ensembler $\ensembler_\vote^w$ is query-access.
\end{theorem}
\begin{proof}
	Let $\certify{F} := \ensembler_\vote^w(\certify{F}^\id{0}, \certify{F}^\id{1}, \ldots, \certify{F}^\id{N-1})$.
	Let $g^\id{0}, g^\id{1}, \ldots, g^\id{N-1} \in \ClassSet \times \{\notrobust, \robust\}$. We use $\bar{g}$ to refer to the set $\{g^\id{0}, g^\id{1}, \ldots, g^\id{N-1}\}$. 
	Define $G$ as follows,
    $$
	 G(g^\id{0}, g^\id{1}, \ldots, g^\id{N-1}) := 
	\big(G_1(g^\id{0}, g^\id{1}, \ldots, g^\id{N-1}), G_2(g^\id{0}, g^\id{1}, \ldots, g^\id{N-1}))
    $$
	where
    $$
	\hat{j} = G_1(g^\id{0}, g^\id{1}, \ldots, g^\id{N-1}) := \argmax_j\Big\{v_{\bar{g}}^w(j,*)\Big\},
    $$
    $$
	G_2(g^\id{0}, g^\id{1}, \ldots, g^\id{N-1}) := \begin{cases}
		1 & \text{if $\forall j \neq \hat{j}$ ~.~ $v_{\bar{g}}^w(\hat{j}, 1) > v_{\bar{g}}^w(*, 0) + v_{\bar{g}}^w(j, 1)$} \\
    0 & \text{otherwise}
    \end{cases}
    $$
	and 
    $$
	v_{\bar{g}}^w(j, c) = \sum_{i = 0}^{N-1} w_i * \indicator\left[g^\id{i} = (j, c)\right]
    $$
	Then, by Def.~\ref{def:voting_ensembler}, $\certify{F}(x) = G(\certify{F}^\id{0}(x), \certify{F}^\id{1}(x), \ldots, \certify{F}^\id{N-1}(x))$.
	Then, by Def.~\ref{def:qa_ensembler}, weighted voting ensembler is query-access.
\end{proof}

\section{Hyper-parameters of Table~\ref{tab:ensemble_attack}}\label{sec:hparams of tabel 1}
In Table~\ref{tab:attack_hparams}, we report hyper-parameters used to run CasA to reach the statistics reported in Table~\ref{tab:ensemble_attack}. Notice that if a normalization is `$\mu=$[0.485, 0.456, 0.406], $\sigma=$0.225', we divide the $\epsilon$ and step size by $\sigma$ during the experiment. We use SGD as the optimizer for all experiments. 

\begin{table}[!t]
    \caption{Hyper-parameters Used for CasA in Table~\ref{tab:ensemble_attack}}
    \label{tab:attack_hparams}
    \centering
    \resizebox{0.8\textwidth}{!}{
    \begin{tabular}{ c c c  c c c}
    \toprule
    \toprule
    Dataset & Norm & $\epsilon$ & Normalization & Max Steps & Step Size  \\
    \midrule
	MNIST & $\ell_\infty$ & 0.1 & [0,1] & 100 & 0.004 \\
	MNIST & $\ell_\infty$ & 0.3 & [0,1] & 100 & 0.012 \\
	MNIST & $\ell_2$ & 1.58 & [0,1] & 100 & 0.03 \\
    \midrule
	CIFAR10 & $\ell_\infty$ & 2/255 & $\mu=$[0.485, 0.456, 0.406], $\sigma=$0.225 & 100 & 0.0003 \\
	CIFAR10 & $\ell_\infty$ & 8/255 & $\mu=$[0.485, 0.456, 0.406], $\sigma=$0.225 & 100 & 0.00124 \\
	CIFAR10 & $\ell_2$ & 36/255 & $\mu=$[0.485, 0.456, 0.406], $\sigma=$0.225 & 100 & 0.0003 \\
    \bottomrule\\
    \end{tabular}}
\end{table}

\section{Weighted Voting Ensemble: Learning Weights}
\label{sec:learn_weights}

The weights $w$ in $\ensembler_\vote^w$ determines the importance of each constituent classifier in the ensemble. 
%Uniform Voting Ensembler (Def.~\ref{def:unif_voting_ensembler}) provides a way to determine $w$ when each model is considered equally important. 
Given a set of $k$ labeled inputs, $S_k$ (e.g. the training set), we would like to learn the optimal weights $w$ that maximize the ensembler's $\cra$ (Def.~\ref{def:cra}) over $S_k$. When $S_k$ resembles the true distribution of the test points, the learned $w$ is expected to be close to the optimal weights that maximizes the $\cra$ of the test set. Weight optimization over $S_k$ naturally takes the following form.
\begin{align}
	\max_{w \in [0, 1]^N} ~~\frac{1}{k}\sum_{(x_i,y_i) \in S_k}\indicator\left[\ensembler_\vote^w(\certify{F}^\id{0}, ..., \certify{F}^\id{N-1})(x_i)=(y_i,\robust) \right]\label{eq:learn-weigths-01-loss}
\end{align}
For the indicator to output 1, it is required that the margin of votes be greater than 0, i.e. $\Delta^w_{x_i}(y_i) := v^w_{x_i}(y_i, 1) - v^w_{x_i}(*, 0) -\max_{j \neq y_i} \{v^w_{x_i}(j, 1)\} >0$. Namely, the votes for the class $y_i$, i.e. $v^w_{x_i}(y_i, 1)$, must be greater than the votes for all other classes i.e. $\max_{j \neq y_i} \{v^w_{x_i}(j, 1)\}$ plus the votes for non-robust predictions $v^w_{x_i}(*, 0)$ as discussed in Def.~\ref{def:voting_ensembler}. 
%We denote the margin by $\Delta^w_{x_i}(y_i) = v^w_{x_i}(y_i, 1) - v^w_{x_i}(*, 0) -\max_{j \neq y_i} \{v^w_{x_i}(j, 1)\}$. 
Eq.(\ref{eq:learn-weigths-01-loss}) then becomes:
\begin{align}
    \max_{w \in [0, 1]^N} ~~\frac{1}{k}\sum_{(x_i,y_i) \in S_k}\indicator\left[\Delta^w_{x_i}(y_i) >0 \right] \label{eq:learn-weigths-hard-margin}
\end{align}
%\paragraph{A Surrogate Objective.} The bottleneck in directly optimizing Eq.(\ref{eq:learn-weigths-hard-margin}) is that the indicator function is not differentiable. We now consider constructing a surrogate version of the indicator function by replacing it with a differentiable and monotonically increasing function $s$. 
The indicator function is not differentiable so we replace it with a differentiable and monotonically increasing function $s$, which leads to Eq.~\ref{eq:learn-weigths-soft-margin}.
\begin{align}
    \max_{w \in [0, 1]^N} ~~\frac{1}{k}\sum_{(x_i,y_i) \in S_k}s(\Delta^w_{x_i}(y_i))  \label{eq:learn-weigths-soft-margin}
\end{align}
%In this paper, we choose $s$ as the sigmoid function $\sigma_{t}$ where $t$ is the temperature only for negative inputs (i.e. we only divide the input with $t$ if it is negative). The use of $\sigma_{t}$ is based on the following reasons: 1) $\sigma_{t}$ is non-negative so margins with different signs will not cancel out each other; 2) $\sigma_{t}$ is upper-bounded by 1 and flattens for large inputs so that the objective will not bias on maximizing the margins for just a few points; and 3) a temperature term $t$ can be used to deal with the vanishing gradients of the sigmoid. However, notice that the flattened regions are actually desired in case of large positive margins as mentioned in (2), so we will only adjust the temperature when the margins $\Delta^w_{x_i}(y_i)$ are negative. Otherwise, the temperature is always set to 1. Finally, we present the optimization objective we solve to obtain an optimal $w^*$ that maximizes the $\cra$ of the ensembler $\ensembler_\vote^w$ on a given set of points $S_k$:
In this paper, we choose $s$ to be the sigmoid function $\sigma_{t}$ where $t$ is the temperature only for negative inputs, i.e., $\sigma_{t}(x) := \sigma(x) \text{ if } x > 0 \text{ and } \sigma(x/t) \text{ otherwise}$, where $\sigma$ is the standard sigmoid function. 
Sigmoid is non-negative so margins with opposite signs do not cancel, and it also avoids biasing training towards producing larger margins on a small number of points.
Indeed, vanishing gradients are useful on points around large positive margins, so the temperature is only applied on negative inputs.
This leads us to Eq.~\ref{eq:learn-weigths-sigmoid-margin}, the optimization objective we solve for optimal weights $w^*$.
\begin{align}
    w^* = \argmax_{w \in [0, 1]^N} ~~\frac{1}{k}\sum_{(x_i,y_i) \in S_k}\sigma_t(\Delta^w_{x_i}(y_i))\label{eq:learn-weigths-sigmoid-margin}
\end{align}

\section{Weighted Voting Ensemble: Empirical Results}
\label{sec:more_eval}

The goal of these experiments is to evaluate the efficacy of our sound voting ensemble.
For our experiments, we use the pre-trained ensemble constituent models made available by \citet{wong2018scaling} to construct three kinds of ensembles, namely, cascading ensembles, uniform voting ensembles, and weighted voting ensembles.
The weights for the weighted voting ensemble are learned in the manner described in Appendix~\ref{sec:learn_weights}. 
We report certified robust accuracy ($\cra$) and standard accuracy ($\acc$) for each ensemble as well as for the best constituent model.
Note that all these ensembles are query-access but only the uniform voting and weighted voting ensembles are soundness-preserving. 
Consequently, the $\cra$ reported for the cascading ensemble grossly overestimates the actual $\cra$ as demonstrated by our attack results.
%By comparing the $\cra$ of the voting ensembles with the cascading ensembles, we are able to get a sense of the extent to which the unsoundness of cascading distorts the $\cra$.
We always set the temperature to 1e5 and learning rate to 1e-2 when learning the weights as described in Appendix~\ref{sec:learn_weights}.

\begin{table}[!t]
    \caption{Results on models pre-trained by \citet{wong2018scaling} for $\ell_\infty$ robustness.}
    \label{tab:l_inf_ensembles}
    \centering
    \resizebox{\textwidth}{!}{
    \begin{tabular}{ c c c  c c | c c | c c | c c }
    \toprule
    \toprule
    {} & {} & {} & \multicolumn{2}{c}{Single Model} & \multicolumn{2}{|c|}{Cascading} & \multicolumn{2}{|c}{Uniform Voting} & \multicolumn{2}{|c}{Weighted Voting} \\
	    Dataset & Model & $\epsilon$ & \cra(\%) & \acc(\%) & \cra(\%) & \acc(\%) & \cra(\%) & \acc(\%) & \cra(\%) & \acc(\%) \\
    \midrule
	    MNIST & Small, Exact%\tablefootnote{Uses \emph{exact} mode of dual networks based local robustness certification} 
	    & 0.1 & 95.54 & 98.96 & 96.33 & 96.62 & 0.01 & 61.68 & 95.54 & 61.68 \\
    MNIST & Small & 0.1 & 94.94 & 98.79 & 96.07 & 96.24 & 9.29 & 65.85 & 94.94 & 98.79 \\
    MNIST & Large & 0.1 & 95.55 & 98.81 & 96.27 & 96.42 & 10.12 & 63.89 & 95.55 & 98.81 \\
    \midrule
    MNIST & Small & 0.3 & 56.21 & 85.23 & 65.41 & 65.80 & 11.48 & 56.46 & 56.21 & 85.23 \\
    MNIST & Large & 0.3 & 58.02 & 88.84 & 65.50 & 65.50 & 26.95 & 65.97 & 58.02 & 88.84 \\
    \midrule
    CIFAR10 & Small & 2/255 & 46.43 & 60.86 & 56.65 & 56.65 & 18.58 & 40.88 & 46.43 & 60.86 \\
    CIFAR10 & Large & 2/255 & 52.65 & 67.70 & 64.88 & 65.14 & 18.07 & 48.92 & 52.65 & 67.70 \\
    \midrule
    CIFAR10 & Small & 8/255 & 20.58 & 27.60 & 28.32 & 28.32 & 10.78 & 24.11 & 19.00 & 23.78 \\
    CIFAR10 & Large & 8/255 & 16.04 & 19.01 & 20.83 & 20.83 & 5.18 & 21.01 & 16.04 & 19.01 \\
    \bottomrule\\
    \end{tabular}}
\end{table}

% ##################################################################################
\begin{table}[!t]
    \caption{Results on models pre-trained by \citet{wong2018scaling} for $\ell_2$ robustness.}
    \label{tab:l_2_ensembles}
    \centering
    \resizebox{\textwidth}{!}{
    \begin{tabular}{ c c c  c c | c c | c c | c c }
    \toprule
    \toprule
    {} & {} & {} & \multicolumn{2}{c|}{Single Model} & \multicolumn{2}{|c|}{Cascading} & \multicolumn{2}{|c}{Uniform Voting} & \multicolumn{2}{|c}{Weighted Voting}\\
	    Dataset & Model & $\epsilon$ & \cra(\%) & \acc(\%)  & \cra(\%)  & \acc(\%)  & \cra(\%)  & \acc(\%)  & \cra(\%)  & \acc(\%)  \\
    \midrule
    MNIST & Small, Exact & 1.58 & 43.52 & 88.14 & 75.58 & 80.43 & 06.42 & 74.25 & 43.52 & 88.14 \\
    MNIST & Small & 1.58 & 43.34 & 87.73 & 74.66 & 79.07 & 06.74 & 74.24 & 43.34 & 87.73 \\
    MNIST & Large & 1.58 & 43.96 & 88.39 & 74.50 & 74.99 & 06.35 & 65.99 & 43.96 & 88.39 \\
    \midrule
    CIFAR10 & Small & 36/255 & 46.05 & 54.39 & 49.89 & 51.37 & 11.47 & 38.74 & 46.05 & 54.39 \\
    CIFAR10 & Large & 36/255 & 50.26 & 60.14 & 58.72 & 58.76 & 10.56 & 39.74 & 50.26 & 60.14 \\
    CIFAR10 & Resnet & 36/255 & 51.65 & 60.7 & 58.65 & 58.69 & 22.74 & 46.71 & 51.65 & 60.7  \\
    \bottomrule\\
    \end{tabular}}
\end{table}

Table ~\ref{tab:l_inf_ensembles} shows the results for $\ell_\infty$ robustness. 
Each row in the table represents a specific combination of dataset (MNIST or CIFAR-10), architecture (Small or Large convolutional networks or Resnet), and $\epsilon$ value used for local robustness certification. 
%The structure of this table is the same as Table 2 in~\citep{wong2018scaling} except for missing rows corresponding to ResNet architecture. 
%The table shows results for a single model, cascading ensemble, uniform voting ensemble, and weighted voting ensemble. 
Table~\ref{tab:l_2_ensembles}  shows the results for $\ell_2$ robustness using constituent models pre-trained by \citet{wong2018scaling}. 
%The structure of this table is the same as Table 4 in~\citep{wong2018scaling} except for the missing rows corresponding to Resnet architecture. We are unable to evaluate the Resnet models as they require more memory than available on our GPU (24GB of RAM). 

\textbf{Summary of Results.}
We see from Tables~\ref{tab:l_inf_ensembles} and~\ref{tab:l_2_ensembles} that while the cascading ensemble appears to improve upon the $\cra$ of the single best model in the ensemble, these numbers are misleading due to the unsoundness of the certification mechanism. The $\cra$ for the uniform voting and weighted ensembles are consistently lower than that reported by the cascading ensemble, and in many cases, significantly so. 
Uniform voting ensembles stand-out for their low $\cra$ but there is a simple explanation for these results. 
The constituent models are trained by \citet{wong2018scaling} in a cascading manner, i.e., later constituent models are trained on only those points that cannot be certified by the previous models. 
This strategy causes the subset of inputs labeled correct and certifiably robust by each constituent model to have minimal overlap. 
%However, as highlighted by the thought experiment in Section~\ref{sec:voting_ensemble}, voting ensembles need these input subsets to strike the right balance between diversity and similarity for improving the $\cra$ . 
However, voting ensembles need these input subsets to strike the right balance between diversity and overlap for improving the $\cra$ . 

Another interesting observation is that, in most cases, the $\cra$ of the weighted voting ensemble and the single best constituent model are the same.
This is again a consequence of the cascaded manner in which the constituent models are trained. The first model in the cascade typically vastly outperforms the subsequent models.
Moreover, as already mentioned, the constituent models have almost no overlap in the input regions where they perform well, and their presence only ends up harming the performance of the voting ensemble. 
As a consequence, the optimal normalized weights, learned by solving the optimization problem described in Appendix~\ref{sec:learn_weights}, typically assign all the mass to the first model.
The detailed weights for each of the weighted voting ensemble are given in Tables~\ref{tab:l_inf_ensembles_weights},~\ref{tab:l_2_ensembles_weights},~\ref{tab:l_inf_new_weights}, and~\ref{tab:l_2_new_weights}.

These results suggest two takeaway messages. 
First, the cascaded strategy of \citet{wong2018scaling} for training constituent models is in conflict with the requirement that constituent models overlap in their behavior for voting ensembles to be effective.
%This, along with our thought experiment, gives up hope that if the constituent models are suitably trained, voting ensembles can improve the $\cra$. We leave this exploration for future work.
This gives up hope that if the constituent models are suitably trained, voting ensembles can improve the $\cra$. We leave this exploration for future work.
Second, even if the constituent models do not show the right balance between diversity and similarity, our weight learning procedure ensures that the performance of the weighted voting ensemble is no worse than the single best constituent model. 
Ideally, we would like the weights to be equally distributed since this conveys that every constituent in the ensemble has something to contribute. 
But, in the worst case, the weights play the role of a model selection procedure, assigning zero weights to constituent models that do not contribute to the ensemble.

\textbf{Non-Sequential Training.}
We conduct another set of experiments where instead of using the constituent models pre-trained by \citet{wong2018scaling}, we train them ourselves in a non-sequential manner.
That is, each constituent model is trained on the entire train dataset, and each constituent only differs due to the randomness of initialization and of stochastic gradient descent during training.
Besides this difference, the code, architecture, hyperparameters, and data used for training are the same as that used by \citet{wong2018scaling}. 
For every combination of dataset, architecture, and $\epsilon$ value, we train three constituent models, and use them to construct cascading, uniform voting, and weighted voting ensembles.

\begin{table}[!t]
    \caption{Results on non-sequentially trained models for $\ell_\infty$ robustness.}
    \label{tab:l_inf_new}
    \centering
    \resizebox{\textwidth}{!}{
    \begin{tabular}{ c c c  c c | c c | c c | c c }
    \toprule
    \toprule
    {} & {} & {} & \multicolumn{2}{c}{Single Model} & \multicolumn{2}{|c|}{Cascading} & \multicolumn{2}{|c}{Uniform Voting} & \multicolumn{2}{|c}{Weighted Voting} \\
	    Dataset & Model & $\epsilon$ & \cra(\%) & \acc(\%) & \cra(\%) & \acc(\%) & \cra(\%) & \acc(\%) & \cra(\%) & \acc(\%) \\
    \midrule
    MNIST & Small, Exact & 0.1 & 95.61 & 99.02 & 97.07 & 99.14 & 95.56 & 99.16 & 95.54 & 98.96 \\
    MNIST & Small & 0.1 & 94.94 & 98.79 & 96.25 & 98.69 & 94.46 & 98.78 & 94.94 & 98.79 \\
    %MNIST & Large & 0.1 & 95.96 & 99.04 & 97.21 & 98.97 & 95.88 & 99.01 & 95.96 & 99.04 \\
    \midrule
    MNIST & Small & 0.3 & 56.21 & 85.23 & 66.09 & 84.77 & 55.24 & 85.02 & 56.21 & 85.23 \\
    %MNIST & Large & 0.3 & 61.21 & 89.08 & 67.00 & 88.84 & 58.16 & 88.30 & 61.21 & 89.08 \\
    \midrule
    CIFAR10 & Small & 2/255 & 46.43 & 60.86 & 55.49 & 62.06 & 43.48 & 62.79 & 42.35 & 61.09 \\
    %CIFAR10 & Large & 2/255 & 52.65 & 67.7 & 58.38 & 65.37 & 36.46 & 63.58 & 32.13 & 56.47 \\
    % CIFAR10 & Resnet & 2/255  \\
    \midrule
    CIFAR10 & Small & 8/255 & 21.04 & 28.29 & 25.11 & 27.73 & 20.44 & 28.32 & 21.04 & 28.29 \\
    %CIFAR10 & Large & 8/255 & 16.69 & 19.32 & 18.67 & 19.02 & 15.96 & 19.21 & 16.69 & 19.32 \\
    % CIFAR10 & Resnet & 8/255  \\
    \bottomrule\\
    \end{tabular}}
\end{table}

\begin{table}[!t]
    \caption{Results on non-sequentially trained models for $\ell_2$ robustness.}
    \label{tab:l_2_new}
    \centering
    \resizebox{\textwidth}{!}{
    \begin{tabular}{ c c c  c c | c c | c c | c c }
    \toprule
    \toprule
    {} & {} & {} & \multicolumn{2}{c|}{Single Model} & \multicolumn{2}{|c|}{Cascading} & \multicolumn{2}{|c}{Uniform Voting} & \multicolumn{2}{|c}{Weighted Voting}\\
	    Dataset & Model & $\epsilon$ & \cra(\%) & \acc(\%)  & \cra(\%)  & \acc(\%)  & \cra(\%)  & \acc(\%)  & \cra(\%)  & \acc(\%)  \\
    \midrule
    MNIST & Small, Exact & 1.58 & 43.52 & 88.14 & 47.42 & 87.78 & 42.71 & 88.19 & 43.52 & 88.14 \\
    MNIST & Small & 1.58 & 43.34 & 87.73 & 47.40 & 87.50 & 42.87 & 87.99 & 43.34 & 87.73 \\
    %MNIST & Large & 1.58 & 43.96 & 88.39 & 46.89 & 88.20 & 43.17 & 88.45 & 43.01 & 88.39 \\
    % \midrule
    CIFAR10 & Small & 36/255 & 46.05 & 54.39 & 52.33 & 55.88 & 37.07 & 57.40 & 42.12 & 54.65 \\
    % CIFAR10 & Large & 36/255 & 59.51 & 66.29 & 60.46 & 60.92 & 46.03 & 62.85 & 59.51 & 66.29 \\
    % CIFAR10 & Resnet & 36/255  \\
    \bottomrule\\
    \end{tabular}}
\end{table}

Table~\ref{tab:l_inf_new} shows the results for $\ell_\infty$ robustness using non-sequentially trained constituent model ands Table~\ref{tab:l_2_new} shows the results for $\ell_2$ robustness.
%Additional results for $\ell_\infty$ robustness as well as for $\ell_2$ robustness are in Appendix~\ref{sec:more_eval} in the supplemental material.
We observe that, for non-sequentially trained models, the $\cra$ of uniform voting and weighted voting ensembles are comparable, and similar to the $\cra$ of the single best constituent model in the ensemble. In this case, the constituent models have too much overlap and almost no diversity. These results reaffirm our observation that voting ensembles require a balance between diversity and similarity to be effective.
%and our thought experiment provides us hope that this is possible in practice.

\begin{table}[!t]
    \caption{Learned weights for weighted voting ensemble with models pre-trained by \citet{wong2018scaling} for $\ell_\infty$ robustness.}
    \label{tab:l_inf_ensembles_weights}
    \centering
    \resizebox{\textwidth}{!}{
    \begin{tabular}{ c c c c | c }
    \toprule
    \toprule
	{} & {} & {} & number of & {} \\
	Dataset & Model & $\epsilon$ & models & weights \\
    \midrule
	MNIST & Small, Exact & 0.1 & 6 & [0.996, 0.003, 0.000, 0.001, 0.000, 0.000] \\
    MNIST & Small & 0.1 & 7 & [0.996, 0.000, 0.000, 0.003, 0.000, 0.000, 0.001] \\
    MNIST & Large & 0.1 & 5 & [0.996, 0.002, 0.001, 0.000, 0.001] \\
    \midrule
    MNIST & Small & 0.3 & 3 & [0.995, 0.003, 0.002] \\
    MNIST & Large & 0.3 & 3 & [0.948, 0.008, 0.044]\\
    \midrule
    CIFAR10 & Small & 2/255  & 5 & [0.995, 0.002, 0.001, 0.001, 0.001]  \\
    CIFAR10 & Large & 2/255 & 4 & [0.994, 0.003, 0.001, 0.002]\\
    % CIFAR10 & Resnet & 2/255  \\
    \midrule
    CIFAR10 & Small & 8/255 & 3 & [0.003, 0.995, 0.002]\\
    CIFAR10 & Large & 8/255 & 3 & [0.995, 0.002, 0.003]\\
    % CIFAR10 & Resnet & 8/255  \\
    \bottomrule\\
    \end{tabular}}
\end{table}

\begin{table}[!t]
    \caption{Learned weights for weighted voting ensembles with models pre-trained by \citet{wong2018scaling} for $\ell_2$ robustness.}
    \label{tab:l_2_ensembles_weights}
    \centering
    \resizebox{\textwidth}{!}{
    \begin{tabular}{ c c c c | c }
    \toprule
    \toprule
	{} & {} & {} & number of & {} \\
	Dataset & Model & $\epsilon$ & models & weights \\
    \midrule
	MNIST & Small, Exact & 1.58 & 6 & [0.995, 0.001, 0.001, 0.000, 0.003, 0.001] \\
    MNIST & Small & 1.58 & 6 & [0.995, 0.001, 0.002, 0.001, 0.001, 0.000] \\
    MNIST & Large & 1.58 & 6 & [0.996, 0.000, 0.001, 0.002, 0.000, 0.001] \\
    \midrule
    CIFAR10 & Small & 36/255  & 2 & [0.994, 0.006] \\
    CIFAR10 & Large & 36/255 & 6 & [0.994, 0.002, 0.001, 0.001, 0.001, 0.001]\\
    CIFAR10 & Resnet & 36/255 & 4 & [0.994, 0.004, 0.001, 0.001]\\
    \bottomrule\\
    \end{tabular}}
\end{table}

\begin{table}[!t]
    \caption{Learned weights for weighted voting ensemble with non-sequentially trained models for $\ell_\infty$ robustness.}
    \label{tab:l_inf_new_weights}
    \centering
    %\resizebox{0.6\textwidth}{!}{
    \begin{tabular}{ c c c c | c }
    \toprule
    \toprule
	{} & {} & {} & number of & {} \\
	Dataset & Model & $\epsilon$ & models & weights \\
    \midrule
	MNIST & Small, Exact & 0.1 & 3 & [0.710, 0.131, 0.159] \\
	MNIST & Small & 0.1 & 3 & [0.694, 0.154, 0.152] \\
    %MNIST & Large & 0.1 & 3 & [0.002, 0.995, 0.003] \\
    \midrule
	MNIST & Small & 0.3 & 3 & [0.908, 0.061, 0.031] \\
    %MNIST & Large & 0.3 & 3 & [0.003, 0.002, 0.995] \\
    \midrule
	CIFAR10 & Small & 2/255 & 3 & [0.011, 0.956, 0.0323]  \\
    %CIFAR10 & Large & 2/255 & 3 & [0.001, 0.004, 0.995]  \\
    \midrule
	CIFAR10 & Small & 8/255 & 3 & [0.042, 0.087, 0.871] \\
    %CIFAR10 & Large & 8/255 & 3 & [0.003, 0.995, 0.003] \\
    \bottomrule\\
    \end{tabular}
    %}
\end{table}

\begin{table}[!t]
    \caption{Learned weights for weighted voting ensembles with  non-sequentially trained models for $\ell_2$ robustness.}
    \label{tab:l_2_new_weights}
    \centering
    %\resizebox{\textwidth}{!}{
    \begin{tabular}{ c c c c | c }
    \toprule
    \toprule
	{} & {} & {} & number of & {} \\
	Dataset & Model & $\epsilon$ & models & weights \\
    \midrule
	MNIST & Small, Exact & 1.58 & 3 & [0.695, 0.159, 0.146] \\
    MNIST & Small & 1.58 & 3 & [0.660, 0.196, 0.144] \\
    %MNIST & Large & 1.58 & 3 & [0.003, 0.995, 0.002] \\
    % \midrule
    CIFAR10 & Small & 36/255  & 3 & [0.003, 0.002, 0.995] \\
    % CIFAR10 & Large & 36/255  & 2 & [0.006, 0.994]  \\
    \bottomrule\\
    \end{tabular}
    %}
\end{table}

\section{An Alternate Formulation of Uniform Voting Ensembler}
\label{sec:permutation_ensembler}

\begin{definition}[Permutation-based Cascading Ensembler]\label{def:permutation_based_cascading_ensembler}
    Let $\certify{F}^\id{0}, \certify{F}^\id{1}, \ldots, \certify{F}^\id{N-1}$ be $N$ certifiable classifiers and $N$ is odd. Suppose $\Pi$ is the set of all permutations of $\{0, 1, \ldots, N-1\}$.
A permutation-based cascading ensembler $\ensembler_\pcas: \mathbb{\certify{F}}^N \rightarrow \mathbb{\certify{F}} $ is defined as follows
    \begin{align*}
	    \ensembler_\pcas(\certify{F}^\id{0}, \certify{F}^\id{1}, \ldots, \certify{F}^\id{N-1})(x) := 
        \begin{cases} 
            	\certify{F}^\id{\pi_0}(x) & \text{if $\exists \pi \in \Pi$. \pcondcert}(\pi) = 1 \\
		(\lab{\certify{F}}^\id{\pi_0}(x),0) & \text{if $\not\exists \pi' \in \Pi$. \pcondcert}(\pi')= 1~\wedge~\exists \pi \in \Pi.~\pcondclean(\pi) = 1\\
		(*,0) & \text{otherwise}
        \end{cases}
        \end{align*}
    where $*$ is a random label selected from $\ClassSet$\footnote{One can also return the plurality prediction of all models for the consideration of clean accuracy but the choice of $*$ will not change the relevant theorems. 
	%(note this in the footnote firstly and will move to the main paragraph later.)
	}
	,$\pi_0$ refers to the first element of the permutation $\pi$,
        \begin{align}
            &\text{\pcondclean}(\pi) := \begin{cases} 
		    1 & \text{if } \exists j.~(\frac{N+1}{2} \leq j \leq N-1) \wedge (\forall i<j.~\lab{\certify{F}}^\id{\pi_i}(x) = \lab{\certify{F}}^\id{\pi_j}(x) ) \\
                0 & \text{otherwise}
            \end{cases}\\
            &\text{\pcondcert}(\pi) := \begin{cases} 
		    1 & \text{if } \exists j.~(\frac{N+1}{2} \leq j \leq N-1) \wedge (\forall i<j.~\certify{F}^\id{\pi_i}(x) = \certify{F}^\id{\pi_j}(x)) \wedge (\cer{\certify{F}}^\id{\pi_j}(x) = \robust) \\
                0 & \text{otherwise}
            \end{cases}
        \end{align} and $i, j \in \{0, 1, ..., N-1\}$.

\end{definition}

\begin{theorem}\label{thm:permutation}
	The permutation-based cascading ensembler $\ensembler_\pcas$ is a soundness-preserving ensembler.
%    If $\certify{F}^\id{0}, \certify{F}^\id{1}, ..., \certify{F}^\id{N-1}$ are sound, $\certify{\pcas}$ is sound. Namely, 
%    \begin{align}
%        \forall x \in \mathbb{R}^d, \certify{\pcas}(x) = \robust  \Longrightarrow \pcas \text{ is } \epsilon\text{-locally robust at } x.
%    \end{align}
\end{theorem}

\begin{proof}
Let $\certify{F} := \ensembler_\pcas(\certify{F}^\id{0}, \certify{F}^\id{1}, \ldots, \certify{F}^\id{N-1})$.
For $\certify{F}$ we want to show that,
\begin{align}
	\forall x \in \mathbb{R}^d, \cer{\certify{F}}(x) = \robust  \Longrightarrow \lab{\certify{F}} \text{ is } \epsilon\text{-locally robust at } x.
\end{align}

	W.L.O.G suppose $ \lab{\certify{F}}(x) = y$. If $ \certify{F}(x) = (y,\robust)$, let us assume that $\pi$ is the permutation such that $\pcondcert(\pi) = 1$. Let $k$ be the integer that makes \pcondcert$(\pi)=1$ to be true. Thus

\begin{align}
	\certify{F}(x) = (y,\robust) \Longrightarrow (\frac{N+1}{2} \leq k \leq N-1) \wedge (\forall i<k, \certify{F}^\id{\pi_i}(x) = \certify{F}^\id{\pi_{k}}(x) = (y, \robust))
\end{align}

By our assumptions that $\certify{F}^\id{0}, \certify{F}^\id{1}, ..., \certify{F}^\id{N-1}$ are sound, which are invariant to the permutation of these models. Therefore, by Def.~\ref{def:soundness-of-F_Tilde}, $\forall i\leq k$,

\begin{align}
	\certify{F}^\id{\pi_i}(x) = (y,1) \Longrightarrow (\forall x'.~ ||x'-x||\leq \epsilon \Longrightarrow \lab{\certify{F}}^\id{\pi_i}(x') = \lab{\certify{F}}^\id{\pi_i}(x) = y) \label{Eq:y}
\end{align}

Eq.~(\ref{Eq:y}) implies that $\forall x' \text{ s.t. } ||x'-x||\leq \epsilon$, the following statement is true

\begin{align}
	(\frac{N+1}{2} \leq k \leq N-1) \wedge (\forall i\leq k, \lab{\certify{F}}^\id{\pi_i}(x) = \lab{\certify{F}}^\id{\pi_k}(x') = y) \label{Eq:x_prime}
\end{align}

Plug the condition~(\ref{Eq:x_prime}) into Def.~\ref{def:permutation_based_cascading_ensembler}, we find that
	$\pcondclean(\pi)=1$ for $x'$. Moreover, there cannot be a permutation $\pi'$ such that $\pcondclean(\pi')=\pcondcert(\pi')=1 \wedge \lab{\certify{F}}^\id{\pi'_0} \neq y$ since $k \geq \frac{N+1}{2}$.
	Therefore, $\lab{\certify{F}}(x')=\lab{\certify{F}}^\id{\pi_0}(x') = y$, and we arrive at the following statement,
\begin{align}
	\forall x'.~||x'-x||\leq \epsilon \Longrightarrow \lab{\certify{F}}(x') = \lab{\certify{F}}(x) = y
\end{align} which completes the proof for the soundness of $\lab{\certify{F}}$ at any $x$.

\end{proof}

%- Point out that Certifiable Robust Ensemble is actually a plurality based ensemble because it relies on that a majority of the model agrees.\\

\end{document}